\documentclass[sigconf,screen,nonacm]{acmart}
\usepackage{dsfont}
\usepackage{url}
\usepackage{amsthm}
\usepackage{afterpage}
\usepackage{subcaption}
\usepackage[normalem]{ulem}
\useunder{\uline}{\ul}{}
\usepackage{graphicx}
\usepackage{textcomp}
\usepackage{xcolor}
\usepackage{balance} 
\usepackage{bbm}
\usepackage{amsmath,amsfonts}
\usepackage{booktabs}
\usepackage{tcolorbox}

\usepackage{epstopdf}
\usepackage{array}
\usepackage{booktabs}
\usepackage{multicol}
\usepackage{color}
\usepackage{xcolor}
\usepackage{colortbl}
\usepackage{xspace}
\usepackage{multirow}
\usepackage{pifont}
\usepackage{enumitem}
\usepackage{bbding}
\usepackage{hyperref}
\usepackage{makecell}
\usepackage{microtype}
\usepackage{fontawesome}
\usepackage{caption,setspace}
\usepackage[utf8]{inputenc}
\usepackage{utfsym}
\usepackage{diagbox}
\usepackage{makecell}
\newcommand{\thickhline}{\noalign{\hrule height 0.3pt}} 

\DeclareMathOperator*{\argmin}{arg\,min}

\AtBeginDocument{%
 }


\usepackage[linesnumbered,ruled,vlined]{algorithm2e}

\newcommand{\ssymbol}[1]{^{\@fnsymbol{#1}}}

\SetKwInput{KwInput}{Input}
\SetKwInput{KwOutput}{Output}
\SetKwRepeat{Do}{do}{while}

\definecolor{darkred}{rgb}{0.55, 0.0, 0.0}
\definecolor{lightred}{rgb}{0.94, 0.5, 0.5}
\definecolor{rosybrown}{rgb}{0.74, 0.56, 0.56}
\definecolor{darkgreen}{rgb}{0.0, 0.39, 0.0}
\definecolor{skyblue}{rgb}{0.56, 0.93, 0.56}
\definecolor{royalblue}{rgb}{0.0, 0.0, 0.55}
\definecolor{lightblue}{rgb}{0.68, 0.85, 0.9}
\definecolor{skyblue}{rgb}{0.53, 0.81, 0.92}

\definecolor{gray}{rgb}{0.5, 0.5, 0.5}
\definecolor{forestgreen}{rgb}{0.13, 0.55, 0.13}
\definecolor{slateblue}{rgb}{0.42, 0.35, 0.80}
\definecolor{darkgray}{rgb}{0.33, 0.33, 0.33}
\definecolor{royalblue}{rgb}{0.25, 0.41, 0.88}

\SetCommentSty{mycommfont}

\usepackage{ulem}

\begin{document}

\title{Towards Adaptive Federated Graph Clustering: A Global Community-aware Contrastive Learning-based Approach}

\author{Yinlin Zhu}
\affiliation{%
\institution{Sun Yat-sen University}
\city{Guangzhou}
\country{China}}
\email{zhuylin27@mail2.sysu.edu.cn}

\author{Di Wu}
\authornote{Corresponding author}
\affiliation{
  \institution{Sun Yat-sen University}
  \city{Guangzhou}
  \country{China}
}
\email{wudi27@mail.sysu.edu.cn}

\author{Wang Luo}
\affiliation{%
\institution{Sun Yat-sen University}
\city{Guangzhou}
\country{China}}
\email{luow69@mail2.sysu.edu.cn}

\author{Guocong Quan}
\affiliation{%
\institution{Sun Yat-sen University}
\city{Guangzhou}
\country{China}}
\email{quangc@mail.sysu.edu.cn}

\author{Miao Hu}
\affiliation{%
\institution{Sun Yat-sen University}
\city{Guangzhou}
\country{China}}
\email{humiao5@mail.sysu.edu.cn}

\renewcommand{\shortauthors}{Yinlin Zhu, Di Wu, Wang Luo, Guocong Quan, Miao Hu}

\begin{abstract}
Federated graph learning (FGL) enables multiple clients to collaboratively train graph models without sharing their private graph data, providing a promising paradigm for mining knowledge from distributed graph repositories. While most existing FGL methods focus on supervised tasks, real-world graphs are often massive and unlabeled, making federated graph clustering an important yet still immature research direction. Notably, this task is particularly challenging due to the inherent subgraph heterogeneity across clients, which leads to client-specific community structures. 
In this work, we identify two critical limitations in existing federated graph clustering methods: (1) unrealistic pre-defined cluster cardinality assumptions and (2) incomplete inter-community separation. To address these challenges, we propose \textbf{AdaFGC}, an \underline{\textbf{Ada}}ptive \underline{\textbf{F}}ederated graph clustering framework based on \underline{\textbf{G}}lobal community-aware \underline{\textbf{C}}ontrastive learning. AdaFGC introduces an over-complete set of global community anchors to model the global community structure and adaptively estimate clustering cardinality via cross-client anchor refinement. In addition, it employs a global community-aware contrastive learning scheme that uses the shared anchors as contrastive prototypes to explicitly enforce community-level attraction and repulsion across clients, complemented by node-level and topology-level objectives that stabilize local representations. 
Extensive experiments on eight benchmark datasets demonstrate that AdaFGC consistently outperforms existing supervised and unsupervised FGL baselines across multiple clustering metrics. 
\end{abstract}

 \maketitle

\section{Introduction}
\label{sec: introduction}

    Recent advances in computational capabilities have catalyzed a data-centric paradigm shift in deep learning, where the community increasingly prioritizes distilling knowledge from expansive and diverse datasets rather than relying solely on architectural innovations~\cite{gpt4, sora}. This trend has naturally extended to the graph domain: the focus is shifting from training graph neural networks (GNNs) on centrally collected graph data to federated graph learning (FGL), which enables multiple clients to collaboratively train graph models by exchanging model updates instead of raw data~\cite{fedpub, fedsage_plus}. As a result, FGL unlocks collective intelligence from distributed graph sources while respecting the privacy constraints and competition concerns inherent in data silos.
    
    Within this landscape, various FGL methods have emerged~\cite{openfgl, data_centric_fgl_survey, fgl_survey}, yet they predominantly target supervised objectives such as node classification~\cite{fgssl, fggp} and graph classification~\cite{gcfl_plus, fedstar}. Meanwhile, real-world graph repositories are often massive and unlabeled, leaving substantial structural and semantic knowledge under-exploited by label-dependent pipelines. To fill this gap, a few recent studies have explored federated graph clustering~\cite{fedgcn, fedncn, fedpka}, which offers a natural unsupervised alternative for mining latent community structures from distributed unlabeled graphs. However, federated graph clustering is fundamentally more challenging than its centralized counterpart. Due to the well-known subgraph heterogeneity in FGL~\cite{fedtad, fgl_unlearning}, local subgraphs across clients typically exhibit non-IID node attribute distributions and topology patterns, further inducing client-specific community structures. Despite recent efforts, existing federated graph clustering methods still suffer from the following two key limitations.
    
    \textbf{L1: Unrealistic Pre-defined Cluster Cardinality.} In realistic FGL settings, each client only observes an unlabeled local subgraph and therefore cannot determine an appropriate local cluster number; similarly, the server cannot determine the global cluster cardinality a priori. Yet existing methods~\cite{fedgcn, fedncn} typically set the global cluster number to the dataset-level ground-truth class count and local cluster numbers to client-level true class counts, which rarely holds in practice and severely limit applicability. Effective mechanisms for adaptively estimating both local and global cluster cardinalities from cross-client knowledge remain absent.
    
    \textbf{L2: Incomplete Inter-community Separation.} Existing federated graph clustering methods primarily focus on discovering similar communities across clients and enhancing their cross-client consistency by aligning node features and topological patterns. For example, FedNCN~\cite{fedncn} matches similar cross-client communities and repairs links between them. However, these approaches largely overlook the explicit separation between dissimilar communities. In contrast, extensive studies on contrastive learning in both vision~\cite{clip, bt_cv_contrastive} and graph domains~\cite{graph_con_1, graph_con_2} have consistently shown that both positive and negative signals are essential for learning discriminative representations. Without sufficient cross-client discriminative guidance from the server, local cluster boundaries remain ambiguous, leading to suboptimal clustering performance.

    Building upon these insights, we propose \textbf{AdaFGC} (\textbf{\underline{Ada}}ptive \textbf{\underline{F}}ederated \textbf{\underline{G}}raph \textbf{\underline{C}}lustering via Global Community-aware Contrastive Learning).
    To address \textbf{L1}, AdaFGC first performs over-complete global community anchor initialization: each client produces over-complete local cluster centroids, and the server aggregates them via a second-stage clustering to form the initial anchor set. In subsequent rounds, each client conducts global community-aware clustering by assigning local nodes to the nearest global community anchors, while the server performs global community anchor evolution, guiding similar anchors toward the same local cluster centroids and merging anchors that become sufficiently close. This adaptive loop progressively refines the clustering cardinality, enabling both the server and clients to discover community structures that better reflect the underlying data distributions.
    To address \textbf{L2}, each client performs global community-aware contrastive learning on its local subgraph, optimizing three complementary objectives at the community level, node level, and topology level. These objectives leverage global community anchors, cross-view node instances, and topological neighborhoods, respectively, to jointly enhance intra-community cohesion and inter-community separation across clients, thereby yielding more discriminative node representations for improved federated clustering.

    \textbf{Our Contributions:} 
    (1) \textbf{Valuable Insights.} We revisit federated graph clustering and identify two core bottlenecks: unrealistic pre-defined cluster cardinality (\textbf{L1}) and incomplete inter-community separation across clients (\textbf{L2}).
    (2) \textbf{Novel Method.} We propose AdaFGC, which introduces over-complete global community anchors with adaptive anchor evolution to estimate clustering cardinality, together with a global community-aware contrastive learning scheme that uses the shared anchors as contrastive prototypes to enforce global community-level attraction and repulsion, complemented by node-level and topology-level objectives to stabilize local representations.
    (3) \textbf{State-of-the-art Performance.} Extensive experiments on eight datasets demonstrate that AdaFGC consistently outperforms existing supervised and unsupervised FGL baselines across four clustering metrics.

\section{Preliminaries and Related Works}
\label{sec: preliminaries}

\textbf{Problem Formulation.} 
We study the problem of federated node-level clustering. Consider a system of $K$ clients, where the $k$-th client holds a private undirected subgraph $G^k=(\mathcal{V}^k, \mathcal{E}^k)$ derived from an implicit global graph $G_g=(\mathcal{V}_g, \mathcal{E}_g)$, i.e., $\mathcal{V}^k \subseteq \mathcal{V}_g$ and $\mathcal{E}^k \subseteq \mathcal{E}_g$. Each node $v_i \in \mathcal{V}^k$ is associated with an attribute vector $\boldsymbol{x}_i^k$ and a ground-truth label $y_i^k$; labels are withheld during training and used solely for evaluation. Under the FedAvg~\cite{fedavg} aggregation strategy, each communication round $t$ proceeds in three stages:
(1) \textit{Client Selection}: the server randomly selects a subset $\mathcal{S}$ of clients to participate in the current round.
(2) \textit{Local Training and Clustering}: each selected client $k$ downloads the global encoder parameters $\hat{\mathbf{W}}$ and obtains node embeddings via the GNN encoder parameterized, after which a clustering module assigns nodes to latent clusters. The client optimizes an unsupervised objective $\mathcal{L}(G^k)$ (e.g., graph reconstruction or contrastive loss) to jointly refine representations and cluster assignments.
(3) \textit{Global Aggregation}: the server collects updated parameters $\{\mathbf{W}^k\}_{k\in \mathcal{S}}$ and computes a weighted average:
$\hat{\mathbf{W}} = \sum_{k\in \mathcal{S}} \frac{|\mathcal{V}^k|}{N} \mathbf{W}^k$,
where $N = \sum_{k\in \mathcal{S}}|\mathcal{V}^k|$ denotes the total number of nodes across selected clients.

\vspace{+0.1cm}
\noindent \textbf{Attributed Graph Clustering.} 
Leveraging the strong representational power of GNNs on graph-structured data, node-level clustering has achieved substantial progress in recent years~\cite{liu2022efficient, li2022high, gong2022deep}. CCGC~\cite{yang2023cluster} constructs dual views of a complete graph and employs a Siamese encoder to guide positive and negative pair generation, thereby improving clustering quality. MAGC~\cite{lin2021multi} exploits complementary multi-view information through adaptive weighting to capture both consistent and discriminative relationships. AMGC~\cite{tu2024attribute} further proposes a unified framework that alternately optimizes clustering and attribute imputation on a single graph. Additional studies can be found in recent surveys~\cite{graph_clustering_survey1, graph_clustering_survey2}.

\vspace{+0.1cm}
\noindent \textbf{Federated Graph Learning.} 
Federated graph learning (FGL) has emerged as a promising paradigm for training GNNs under data-privacy constraints~\cite{openfgl, fgl_survey, data_centric_fgl_survey, fedgfm}. Most existing FGL methods focus on supervised tasks such as node classification and graph classification. FedSage+~\cite{fedsage_plus} designs a missing-neighbor generator using node labels to address the missing-link problem in distributed subgraph systems. FedPUB~\cite{fedpub} constructs a global random graph to obtain multiple functional embeddings and alleviate multi-source heterogeneity. FedTAD~\cite{fedtad} evaluates the reliability of node-class knowledge in a topology-aware manner and guides pseudo-graph generation for improved classification. Recent methods~\cite{fedspa, fedssp, power} have also demonstrated competitive performance.

However, real-world graph repositories are typically large-scale and sparsely labeled, leaving considerable structural and semantic information under-exploited by label-dependent pipelines. Motivated by this gap, a few recent studies have begun to explore federated graph clustering. FedGCN~\cite{fedgcn} proposes an unsupervised framework for federated graph-level clustering that uploads structure-oriented cluster prototypes and generates global consensus prototypes via Gaussian estimation. FedNCN~\cite{fedncn} focuses on federated node-level clustering and introduces a clustering projector that produces privacy-preserving counterparts while retaining cluster-specific characteristics. Nevertheless, these approaches remain limited by \textit{unrealistic pre-defined cluster cardinality} and \textit{incomplete inter-community separation}, as discussed in Sec.~\ref{sec: introduction}.

\begin{figure*}[htb]
 \centering
  \includegraphics[width=0.998\textwidth]{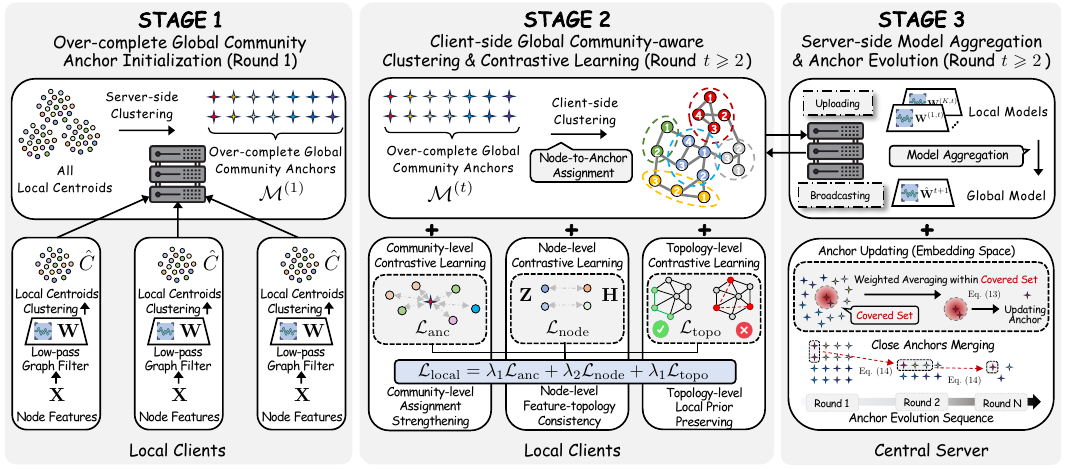}
  \caption{Overview of the AdaFGC framework. AdaFGC maintains an over-complete yet adaptive global community anchor space to handle unknown cluster cardinality, and optimizes cluster boundaries with global community-aware contrastive signals.}
\label{fig: framework}
\end{figure*}

\section{Methodology}
\label{sec: method}

In this section, we present \textbf{AdaFGC}, a federated graph clustering framework designed to address the two bottlenecks identified in Sec.~\ref{sec: introduction}: unrealistic pre-defined cluster cardinality (\textbf{L1}) and incomplete inter-community separation (\textbf{L2}). An overview of AdaFGC is provided in Fig.~\ref{fig: framework}. We detail the three core components below: over-complete global community anchor initialization (Sec.~\ref{sec: Over-complete Global Community Anchor Initialization}), global community-aware clustering and contrastive learning (Sec.~\ref{sec: Local Clustering and Contrastive Training}), and server-side anchor evolution and model aggregation (Sec.~\ref{sec: Global Community Anchor Evolution}).

\subsection{Over-complete Global Community Anchor Initialization}
\label{sec: Over-complete Global Community Anchor Initialization}

To tackle \textbf{L1}, AdaFGC avoids committing to a fixed cluster cardinality at either the client or server side. Instead, in the first communication round, we initialize an over-complete set of global community anchors and let subsequent rounds adaptively prune redundancy. Let $\hat{C}$ denote the initial anchor count, where $\hat{C}\gg C$ and $C$ is the unknown true number of communities. By choosing $\hat{C}$ sufficiently large, we ensure coverage of all potential communities while deferring the elimination of redundant anchors to later rounds.

Specifically, for client $k$ with local graph $G^k=(\mathcal{V}^k,\mathcal{E}^k)$ and feature matrix $\mathbf{X}^k\in\mathbb{R}^{|\mathcal{V}^k|\times F}$, we first encode the node features to obtain feature-view embeddings $\mathbf{Z}^k$, and then apply low-pass graph filtering to produce topology-view embeddings $\mathbf{H}^k$:
\begin{equation}
    \label{eq: forward}
    \mathbf{Z}^k = \mathbf{X}^k\mathbf{W}^k, \quad \mathbf{H}^k = \frac{1}{\alpha+1}\sum_{l=0}^L\Big(\frac{\alpha}{\alpha+1} \tilde{\mathbf{A}}\Big)^l \mathbf{Z}^k,
\end{equation}
where $\mathbf{W}^k \in \mathbb{R}^{F\times d}$ denotes local encoder parameters initialized from the broadcast global encoder $\hat{\mathbf{W}}^{(1)}$, $F$ and $d$ are the input and latent feature dimensions respectively, $\alpha>0$ controls smoothing strength, $\tilde{\mathbf{A}}$ is the symmetrically normalized adjacency matrix, and $L$ is the propagation depth. As we show in Theorem~\ref{thm: smooth fused rep}, many widely used GNNs (e.g., GCN, SGC~\cite{gcn, sgc}) can be viewed as special cases of this graph-filtering formulation.

We then apply K-Means~\cite{mcqueen1967some} on $\mathbf{H}^k$ to produce $\hat{C}$ local centroids:
\begin{equation}
\label{eq: local_centroids}
\mathcal{C}^k=\{\mathbf{c}^{k}_1,\ldots,\mathbf{c}^{k}_{\hat{C}}\}.
\end{equation}
All $K$ clients upload their local centroid sets to the server, forming a global pool of $K\hat{C}$ candidates. The server then performs a second-stage clustering on this pool to initialize $\hat{C}$ global community anchors, which can be formulated as:
\begin{equation}
\label{eq: global_anchors_init}
\mathcal{M}^{(1)}=\{\mathbf{m}_1^{(1)},\ldots,\mathbf{m}_{\hat{C}}^{(1)}\}.
\end{equation}
This over-complete initialization intentionally favors coverage over precision; subsequent rounds progressively merge similar anchors using multi-client clustering signals (Sec.~\ref{sec: Global Community Anchor Evolution}).

\subsection{Client Side: Global Community-aware Clustering and Contrastive Learning}
\label{sec: Local Clustering and Contrastive Training}

From round $t \geqslant 2$, AdaFGC coordinates clients and the server for local clustering, local optimization, and server-side anchor evolution. We first describe the client-side procedure.

\vspace{+0.1cm}
\noindent \textbf{Global Community-aware Clustering.} The server broadcasts the global encoder weights $\hat{\mathbf{W}}^{(t)}$ and anchor set $\mathcal{M}^{(t)}=\{\mathbf{m}^{(t)}_1,\ldots,\mathbf{m}^{(t)}_{\hat{C}}\}$ to each client. Client $k$ computes $\mathbf{Z}^k$ and $\mathbf{H}^k$ via Eq.~\eqref{eq: forward} and assigns each node $v_i \in \mathcal{V}^k$ to its nearest global community anchor:
\begin{equation}
\label{eq: node_assign}
a_i^k=\operatorname*{argmin}_{j\in\{1,\ldots,\hat{C}\}}\|\mathbf{h}^{k}_i-\mathbf{m}^{(t)}_j\|_2^2.
\end{equation}

The set of nodes assigned to the $j$-th anchor is denoted:
\begin{equation}
\label{eq: cluster_set}
\mathcal{B}^{k}_j=\{v_i\in\mathcal{V}^k\mid a_i^k=j\},
\end{equation}
and the corresponding local cluster centroid is computed as:
\begin{equation}
\label{eq: local_centroid}
\mathbf{c}^{k}_j=
\begin{cases}
\frac{1}{|\mathcal{B}_j^k|}\sum_{i\in\mathcal{B}^{k}_{j}}\mathbf{h}_i^k, & |\mathcal{B}^{k}_{j}|>0,\\
\mathbf{m}^{(t)}_j, & |\mathcal{B}^{k}_{j}|=0,
\end{cases}
\end{equation}
where $\mathbf{c}_j^k$ serves as a client-specific update signal for the $j$-th anchor, encoding how local feature and topological patterns reshape the shared global community structure.

\vspace{+0.1cm}
\noindent \textbf{Global Community-aware Contrastive Learning.} 
To address \textbf{L2}, we design a contrastive learning scheme centered on the shared global community anchors. The key idea is to explicitly enforce community-level attraction and repulsion via the anchors, while maintaining local representation consistency through auxiliary contrastive objectives. Concretely, we optimize three complementary losses operating at different granularities.

(1) \textit{Community-level Contrastive Loss.}
For each anchor $j$, the assigned nodes form the set $\mathcal{B}_j^k$, from which we select the nearest $\beta\%$ nodes (denoted $\mathcal{T}_j^k \subset \mathcal{B}_j^k$) as high-confidence positives. The assigned anchor $\mathbf{m}_j^{(t)}$ serves as the positive prototype, while the remaining anchors act as negatives. This objective injects community-level supervision from the global structure, pulling same-community nodes toward their anchor while repelling different communities apart. Specifically, letting $\psi(\mathbf{u},\mathbf{v})=\exp(\mathrm{cos}(\mathbf{u},\mathbf{v})/\tau)$ with temperature $\tau$, the community-level loss for anchor $j$ is:
\begin{equation}
\label{eq: loss_anc}
\mathcal{L}_{\mathrm{anc}}^{(j)}=-\sum_{v_i\in\mathcal{T}_j^k}\log 
\frac{\psi(\mathbf{h}_i^k,\mathbf{m}_j^{(t)})}
{\psi(\mathbf{h}_i^k,\mathbf{m}_j^{(t)})+\sum_{j'\neq j}\psi(\mathbf{m}_j^{(t)},\mathbf{m}_{j'}^{(t)})}.
\end{equation}
The full community-level loss is $\mathcal{L}_{\mathrm{anc}}=\sum_{j=1}^{\hat{C}}\mathcal{L}_{\mathrm{anc}}^{(j)}$.

(2) \textit{Node-level Contrastive Loss.}
To stabilize representations, we enforce instance-level consistency between the feature view $\mathbf{Z}^k$ and the topology view $\mathbf{H}^k$. For each node $v_i$, this objective pulls together the two views of the same node while contrasting them against other nodes, which is formulated as follows:
\begin{equation}
\label{eq: loss_node}
\mathcal{L}_{\mathrm{node}}^{(i)}=-\log 
\frac{\psi(\mathbf{z}_i^k,\mathbf{h}_i^k)}
{\sum_{v_j\in\mathcal{V}^k}\psi(\mathbf{z}_i^k,\mathbf{h}_j^k)},
\end{equation}
and the full node-level loss is $\mathcal{L}_{\mathrm{node}}=\sum_{v_i\in\mathcal{V}^k}\mathcal{L}_{\mathrm{node}}^{(i)}$.

(3) \textit{Topology-level Contrastive Loss.}
To further preserve local structural knowledge, we introduce a topology-level objective based on random walks. For each node $v_i$, the visited set $\mathcal{P}_i^k=\{p_1,\ldots,p_r\}$ serves as structural positives, while $\mathcal{N}_i^k=\{n_1,\ldots,n_r\}$ is sampled from $\mathcal{V}^k\setminus\mathcal{P}_i^k$ as negatives, formulated as:
\begin{equation}
\label{eq: loss_topo}
\mathcal{L}_{\mathrm{topo}}^{(i)}=-\log 
\frac{\sum_{v_j\in\mathcal{P}_i^k}\psi(\mathbf{h}_i^k,\mathbf{h}_j^k)}
{\sum_{v_j\in\mathcal{P}_i^k}\psi(\mathbf{h}_i^k,\mathbf{h}_j^k)+\sum_{v_j\in\mathcal{N}_i^k}\psi(\mathbf{h}_i^k,\mathbf{h}_j^k)}.
\end{equation}

\noindent The full topology-level loss is $\mathcal{L}_{\mathrm{topo}}=\sum_{v_i\in\mathcal{V}^k}\mathcal{L}_{\mathrm{topo}}^{(i)}$.

Finally, the overall local objective combines the three losses:
\begin{equation}
\label{eq: loss_local}
\mathcal{L}_{\mathrm{local}}=
\lambda_1\mathcal{L}_{\mathrm{anc}}+
\lambda_2\mathcal{L}_{\mathrm{node}}+
\lambda_3\mathcal{L}_{\mathrm{topo}},
\end{equation}
where $\lambda_1$, $\lambda_2$, and $\lambda_3$ are balancing coefficients for the community-level, node-level, and topology-level contrastive losses, respectively.
After local optimization, each client uploads the assignment counts $\{|\mathcal{B}^{k}_j|\}_{j=1}^{\hat{C}}$, the local centroids $\{\mathbf{c}^{k}_j\}_{j=1}^{\hat{C}}$, and the updated parameters $\mathbf{W}^{k,(t)}$ to the server for anchor evolution (Sec.~\ref{sec: Global Community Anchor Evolution}).

\subsection{Server Side: Model Aggregation and Global Community Anchor Evolution}
\label{sec: Global Community Anchor Evolution}

The over-complete initialization in Sec.~\ref{sec: Over-complete Global Community Anchor Initialization} ensures sufficient coverage but inevitably introduces redundancy, since $\hat{C} \gg C$. Consequently, nodes belonging to the same community may be split across multiple anchors due to minor pattern variations.
To progressively resolve \textbf{L1}, the server refines the anchors using the local centroids and assignment statistics uploaded by the clients.

\vspace{+0.1cm}
\noindent \textbf{Model Aggregation.}
The server first aggregates encoder parameters from participating clients via specific FL aggregation strategy. Without loss of generality, we take FedAvg~\cite{fedavg} as an example:
\begin{equation}
\label{eq: fedavg}
\hat{\mathbf{W}}^{(t+1)}=
\sum_{k\in\mathcal{S}}
\frac{|\mathcal{V}^k|}
{\sum_{l\in\mathcal{S}}|\mathcal{V}^{l}|}
\mathbf{W}^{k,(t)},
\end{equation}
where $\mathcal{S}$ denotes the participating client set at round $t$.

\vspace{+0.1cm}
\noindent \textbf{Global Community Anchor Evolution.} 
Each client uploads centroids $\{\mathbf{c}^{k}_j\}_{j=1}^{\hat{C}}$ and assignment sizes $\{|\mathcal{B}^{k}_j|\}_{j=1}^{\hat{C}}$ for the current anchor set $\mathcal{M}^{(t)}=\{\mathbf{m}_1^{(t)},\ldots,\mathbf{m}_{\hat{C}}^{(t)}\}$. The server evolves the anchors with two objectives:
(1) each anchor gradually captures the underlying distribution of local graph data based on multi-client clustering results;
(2) anchors corresponding to the same community progressively converge and are merged when sufficiently similar, thereby eliminating redundancy. Formally, for each local centroid $\mathbf{c}^{k}_j$ with assigned anchor $\mathbf{m}_j^{(t)}$, we define its covered anchor set:
\begin{equation}
\label{eq: covered_set}
\mathcal{P}(\mathbf{m}_j^{(t)}  | \mathbf{c}_j^k) = \{\mathbf{m}_p^{(t)}\in \mathcal{M}^{(t)}\mid\|\mathbf{m}_p^{(t)}-\mathbf{m}_j^{(t)}\|_2 \leqslant \|\mathbf{c}^{k}_j-\mathbf{m}_j^{(t)}\|_2\},
\end{equation}
which includes all anchors close enough to the local centroid to be jointly aligned. The directly assigned anchor $\mathbf{m}_j^{(t)}$ is always included in $\mathcal{P}(\mathbf{m}_j^{(t)} | \mathbf{c}_j^k)$. Letting $w_j^k = |\mathcal{B}_j^k|$, anchor $\mathbf{m}_j^{(t)}$ is updated as a weighted average over all centroids whose covered sets contain it, which can be formulated as follows:

\begin{equation}
\label{eq: anchor_update}
\mathbf{m}_j^{(t)} \;\; \leftarrow \;\;
(1-\eta)\, \mathbf{m}_j^{(t)} \; + \;
\eta \, 
\frac{
\sum_{k=1}^K \sum_{p=1}^{\hat{C}} \mathds{1}\big(\mathbf{m}_j^{(t)} \in \mathcal{P}(\mathbf{m}_p^{(t)}|\mathbf{c}_p^k)\big) \, w_p^k \, \mathbf{c}_p^k
}{
\sum_{k=1}^K \sum_{p=1}^{\hat{C}} \mathds{1}\big(\mathbf{m}_j^{(t)} \in \mathcal{P}(\mathbf{m}_p^{(t)}|\mathbf{c}_p^k)\big) \, w_p^k
}.
\end{equation}
To avoid order effects, all anchor updates are applied synchronously.

Finally, to complete the resolution of \textbf{L1}, the server merges any pair of anchors that are sufficiently close. Specifically, anchors $\mathbf{m}_i$ and $\mathbf{m}_j$ are merged if $
\|\mathbf{m}_i - \mathbf{m}_j\|_2 < \epsilon$, where $\epsilon$ controls the merging sensitivity, which can be formulated as follows:
\begin{equation}
\label{eq: anchor_merge}
\mathbf{m}_{i \oplus j} = \frac{\mathbf{m}_i + \mathbf{m}_j}{2}.
\end{equation}
Each merge reduces $\hat{C}$ by one. Through iterative client-side clustering and optimization, coupled with server-side aggregation and anchor evolution, AdaFGC progressively refines the initially over-complete anchor pool into a compact and discriminative global community structure. The complete procedure of AdaFGC is presented in Algorithm~\ref{alg: adafgc}.

\normalem
\begin{algorithm}[htbp]\fontsize{8pt}{4pt}\selectfont
\DontPrintSemicolon
\SetAlgoLined
\caption{Overall Procedure of AdaFGC}
\label{alg: adafgc}

\KwInput{
$K$ clients with local graphs $\{G^k=(\mathcal{V}^k,\mathcal{E}^k)\}_{k=1}^K$ and features $\{\mathbf{X}^k\}$;
initial anchor count $\hat{C}\gg C$;
communication rounds $T$;
local epochs $E$;
smoothing parameter $\alpha$; propagation depth $L$;
loss weights $\lambda_1,\lambda_2,\lambda_3$;
confidence ratio $\beta$; merge threshold $\epsilon$; momentum $\eta$.
}

\KwOutput{Node embeddings $\{\mathbf{H}^k\}_{k=1}^K$ and global community anchors $\mathcal{M}$.}

\tcc{Phase 1: Over-complete Global Community Anchor Initialization (Round $t=1$)}
Server initializes encoder $\hat{\mathbf{W}}^{(1)}$ and broadcasts to all clients\;
\For{each client $k=1,\ldots,K$ \textbf{in parallel}}{
    compute $\mathbf{Z}^k, \mathbf{H}^k$ via Eq.~\eqref{eq: forward}\;
    $\mathcal{C}^k = \text{K-Means}(\mathbf{H}^k, \hat{C})$ via Eq.~\eqref{eq: local_centroids} \tcp*{$\hat{C}$ local centroids}
    upload $\mathcal{C}^k$ to server\;
}
Server clusters $\bigcup_{k=1}^K \mathcal{C}^k$ ($K\hat{C}$ candidates) into $\hat{C}$ global anchors
$\mathcal{M}^{(1)}$ via Eq.~\eqref{eq: global_anchors_init}\;

\tcc{Phase 2: Iterative Federated Clustering and Anchor Evolution (Rounds $t\geqslant 2$)}
\For{$t=2$ to $T$}{
    Server broadcasts $\hat{\mathbf{W}}^{(t)}$ and $\mathcal{M}^{(t)}$ to all clients\;

    \tcc{Client-side: Clustering and Contrastive Optimization}
    \For{each client $k\in\mathcal{S}$ \textbf{in parallel}}{
        $\mathbf{W}^k \leftarrow \hat{\mathbf{W}}^{(t)}$\;
        \For{$e=1$ to $E$}{
            compute $\mathbf{Z}^k,\mathbf{H}^k$ via Eq.~(\ref{eq: forward})\;
            \tcc{Global Community-aware Clustering}
            \For{each node $v_i\in\mathcal{V}^k$}{
                assign $a_i^k$ via Eq.~\eqref{eq: node_assign}\;
            }
            compute cluster sets $\{\mathcal{B}_j^k\}$ via Eq.~\eqref{eq: cluster_set} and local centroids $\{\mathbf{c}_j^k\}$ via Eq.~\eqref{eq: local_centroid}\;
            \tcc{Global Community-aware Contrastive Learning}
            select high-confidence positives $\mathcal{T}_j^k$ (top $\beta$\% nearest nodes per anchor)\;
            compute $\mathcal{L}_{\mathrm{anc}}$ via Eq.~\eqref{eq: loss_anc}: community-level contrastive loss\;
            compute $\mathcal{L}_{\mathrm{node}}$ via Eq.~\eqref{eq: loss_node}: node-level contrastive loss between $\mathbf{Z}^k$ and $\mathbf{H}^k$\;
            sample random-walk positives $\mathcal{P}_i^k$ and negatives $\mathcal{N}_i^k$\;
            compute $\mathcal{L}_{\mathrm{topo}}$ via Eq.~\eqref{eq: loss_topo}: topology-level contrastive loss\;
            $\mathcal{L}_{\mathrm{local}} = \lambda_1\mathcal{L}_{\mathrm{anc}} + \lambda_2\mathcal{L}_{\mathrm{node}} + \lambda_3\mathcal{L}_{\mathrm{topo}}$ via Eq.~\eqref{eq: loss_local}\;
            update $\mathbf{W}^k$ by minimizing $\mathcal{L}_{\mathrm{local}}$\;
        }
        upload $\{|\mathcal{B}_j^k|\}_{j=1}^{\hat{C}}$, $\{\mathbf{c}_j^k\}_{j=1}^{\hat{C}}$, $\mathbf{W}^k$ to server\;
    }

    \tcc{Server-side: Model Aggregation}
    $\hat{\mathbf{W}}^{(t+1)} = \sum_{k\in\mathcal{S}} \frac{|\mathcal{V}^k|}{\sum_{l\in\mathcal{S}}|\mathcal{V}^l|} \mathbf{W}^{k}$ via Eq.~\eqref{eq: fedavg} \tcp*{FedAvg}

    \tcc{Server-side: Global Community Anchor Evolution}
    \For{each anchor $\mathbf{m}_j^{(t)}\in\mathcal{M}^{(t)}$}{
        compute covered sets $\mathcal{P}(\mathbf{m}_j^{(t)}|\mathbf{c}_p^k)$ via Eq.~\eqref{eq: covered_set} for all clients $k$ and centroids $p$\;
        update $\mathbf{m}_j^{(t)}$ via Eq.~\eqref{eq: anchor_update}
        \tcp*{synchronous update}
    }

    \tcc{Server-side: Anchor Merging}
    \For{each pair $(\mathbf{m}_i, \mathbf{m}_j)$ with $\|\mathbf{m}_i - \mathbf{m}_j\|_2 < \epsilon$}{
        merge $\mathbf{m}_{i\oplus j} = (\mathbf{m}_i + \mathbf{m}_j) / 2$ via Eq.~\eqref{eq: anchor_merge}; \quad
        $\hat{C} \leftarrow \hat{C} - 1$\;
    }
    update $\mathcal{M}^{(t+1)}$\;
}

\Return{$\{\mathbf{H}^k\}_{k=1}^K$, $\mathcal{M}^{(T)}$}\;

\end{algorithm}
\ULforem

\section{Experiments}

In this section, we first describe the experimental setup, with full reproducibility details deferred to Appendix~\ref{appendix: dataset details} and Appendix~\ref{appendix: more experimental setups}. 
We then conduct comprehensive empirical evaluations to answer the following research questions:
\textbf{Q1}: Does AdaFGC outperform existing methods for federated graph clustering?
\textbf{Q2}: How well does AdaFGC reveal meaningful community structures across clients?
\textbf{Q3}: What are the individual contributions of each component in AdaFGC?
\textbf{Q4}: How sensitive is AdaFGC to key hyperparameters?
\textbf{Q5}: What are the time and memory overheads, and how efficient is convergence?
\textbf{Q6}: How robust is AdaFGC under noisy settings? 
Moreover, additional results on varying numbers of participating clients and non-graph baselines are provided in Appendix~\ref{appendix: more experiments}.

\subsection{Experimental Setup}
\label{sec: exp setup}

\begin{table}[htbp]
    \setlength{\abovecaptionskip}{0.2cm}
    \setlength{\belowcaptionskip}{-0.2cm}
    \centering
    \caption{\textbf{Statistics of the used datasets.}}
    \label{tab: ablation}
    \footnotesize 
    \renewcommand{\arraystretch}{1.1}
    \resizebox{\linewidth}{!}{
    \setlength{\tabcolsep}{2.5mm}{
    \label{tab: datasets}
\begin{tabular}{lccccc}
\toprule
\textbf{Dataset} & \textbf{Nodes} & \textbf{Edges} & \textbf{Dimensions} & \textbf{Classes} & \textbf{Domain} \\ \midrule
CiteSeer        & 3,327  & 9,104   & 3,703 & 6   & Citation         \\
PubMed          & 19,717 & 88,648  & 500  & 3   & Citation         \\
Amazon-Computer & 13,752 & 491,722 & 767  & 10  & Product \\
Amazon-Photo    & 7,650  & 238,162 & 754  & 8  & Product    \\
Questions       & 48,921 & 153,540 & 301  & 2  & Q\&A System      \\ 
ogb-arxiv      &  169,343 & 1,166,243 &  128 & 40 & Citation \\
ogb-products   & 2,449,029 & 61,859,140 & 100 & 47 & Product\\
Reddit &  232,965 & 23,213,838 &  602 & 41 & Social \\
\bottomrule

\end{tabular}
    }}
\end{table}

\vspace{+0.1cm}
\noindent \textbf{Datasets and Simulation Strategy.}
We evaluate AdaFGC on eight benchmark datasets across diverse domains: three citation networks (CiteSeer, PubMed~\cite{Yang16cora}, ogb-arxiv~\cite{hu2020ogb}), three co-purchase networks (Amazon-Computer, Amazon-Photo~\cite{shchur2018amazon_datasets}, ogb-products~\cite{hu2020ogb}), one Q\&A network (Questions~\cite{platonov2023hete_gnn_survey4}), and one social network (Reddit~\cite{hu2020ogb}). For federated node clustering, we split each graph into 10 client subgraphs using the Louvain algorithm~\cite{louvain}, a standard modularity-based method in subgraph FL benchmarks~\cite{openfgl}. Table~\ref{tab: datasets} summarizes dataset statistics, and Appendix~\ref{appendix: dataset details} provides further details.

\vspace{+0.1cm}
\noindent \textbf{Baseline Methods.}
We compare AdaFGC against two groups of baselines:
(1)~\textit{Supervised FGL methods}: FedSage+~\cite{fedsage_plus}, FedPUB~\cite{fedpub}, FedTAD~\cite{fedtad}, FedGTA~\cite{fedgta}, FedIIH~\cite{fediih}, FGSSL~\cite{fgssl}, FGGP~\cite{fggp}, FedSPA~\cite{fedspa}, and S2FGL~\cite{s2fgl};
(2)~\textit{Unsupervised FGL methods}: FedNCN~\cite{fedncn}. Notably, for supervised FGL methods, when label-dependent computations are required, we replace the true labels with pseudo labels obtained from local clustering.

\vspace{+0.1cm}
\noindent \textbf{Clustering Metrics.}
We adopt four standard clustering evaluation metrics: Clustering-Accuracy (ACC), Normalized Mutual Information (NMI), Adjusted Rand Index (ARI), and F1 Score (F1). Formal definitions are given in Appendix~\ref{appendix: evaluation metrics}.

\begin{table*}[htbp]
    \setlength{\abovecaptionskip}{0.2cm}
    \setlength{\belowcaptionskip}{-0.2cm}
    \centering
    \caption{\textbf{Node clustering performance} on five datasets (10 clients). The best, second best and third best results are highlighted in \textcolor{darkred}{\textbf{red}}, \textcolor{royalblue}{\textbf{blue}} and \textcolor{orange}{\textbf{orange}}, respectively. `*' denotes the unsupervised adaptions for supervised algorithms (refer to our baseline setting in Sec.~\ref{sec: exp setup}).}
    \label{tab: main}
    \footnotesize 
    \renewcommand{\arraystretch}{1.1}
    \resizebox{\linewidth}{!}{
    \setlength{\tabcolsep}{1.2mm}{
    \begin{tabular}{c!{\vrule width 0.1pt}
    cc!{\vrule width 0.1pt}
    cc!{\vrule width 0.1pt}
    cc!{\vrule width 0.1pt}
    cc!{\vrule width 0.1pt}
    cc}
    \hline\thickhline
    \rowcolor{gray!10}
    
    & \multicolumn{2}{c!{\vrule width 0.1pt}}{\textbf{CiteSeer}}  
    & \multicolumn{2}{c!{\vrule width 0.1pt}}{\textbf{PubMed}}  
    & \multicolumn{2}{c!{\vrule width 0.1pt}}{\textbf{Amazon-Computer}}  
    & \multicolumn{2}{c!{\vrule width 0.1pt}}{\textbf{Amazon-Photo}}  
    & \multicolumn{2}{c}{\textbf{Questions}}  \\
    
    \cline{2-11}
    \rowcolor{gray!10}
    \multirow{-2}{*}{\diagbox[width=8.5em,height=2.4em]{\textbf{Methods}}{\textbf{Datasets}}}  
    & ACC & NMI 
    & ACC & NMI 
    & ACC & NMI 
    & ACC & NMI 
    & ACC & NMI \\
    
    \hline
    FedSage+$^*$ & $16.94_{\pm 2.06}$ & $5.53_{\pm 0.28}$ & $40.57_{\pm 7.89}$ & $2.18_{\pm 4.97}$ & \textcolor{orange}{$\mathbf{22.61_{\pm 0.82}}$} & \textcolor{orange}{$\mathbf{8.19_{\pm 0.82}}$} & \textcolor{orange}{$\mathbf{34.35_{\pm 1.30}}$} & $11.22_{\pm 1.28}$ & $77.29_{\pm 1.48}$ & \textcolor{orange}{$\mathbf{0.83_{\pm 1.56}}$} \\
    \rowcolor{gray!10}
    FedPUB$^*$ & $8.86_{\pm 5.23}$ & $0.00_{\pm 0.01}$ & $31.38_{\pm 7.72}$ & $0.00_{\pm 0.01}$ & $22.60_{\pm 10.20}$ & $0.00_{\pm 0.01}$ & $8.24_{\pm 3.44}$ & $0.00_{\pm 0.01}$ & $61.13_{\pm 44.15}$ & $0.00_{\pm 0.01}$ \\
    FedTAD$^*$ & $18.82_{\pm 1.38}$ & $4.03_{\pm 2.76}$ & $36.85_{\pm 4.43}$ & $0.38_{\pm 0.51}$ & $13.30_{\pm 11.94}$ & $3.82_{\pm 2.29}$ & $13.41_{\pm 3.89}$ & $5.59_{\pm 1.84}$ & $67.10_{\pm 32.03}$ & $0.10_{\pm 0.06}$ \\
    \rowcolor{gray!10}
    FedGTA$^*$ & \textcolor{orange}{$\mathbf{27.16_{\pm 0.66}}$} & \textcolor{orange}{$\mathbf{8.42_{\pm 1.66}}$} & \textcolor{orange}{$\mathbf{47.53_{\pm 1.77}}$} & $2.15_{\pm 0.81}$ & $20.75_{\pm 1.52}$ & $6.19_{\pm 1.10}$ & $33.46_{\pm 1.24}$ & \textcolor{orange}{$\mathbf{12.15_{\pm 0.77}}$} & \textcolor{orange}{$\mathbf{88.31_{\pm 3.08}}$} & $0.38_{\pm 0.07}$ \\
    FedIIH$^*$ & $18.08_{\pm 6.58}$ & $0.67_{\pm 0.72}$ & $33.70_{\pm 7.42}$ & $0.02_{\pm 0.03}$ & $7.44_{\pm 4.78}$ & $0.09_{\pm 0.11}$ & $11.52_{\pm 6.71}$ & $0.00_{\pm 0.01}$ & $79.38_{\pm 3.66}$ & $0.00_{\pm 0.01}$ \\
    \rowcolor{gray!10}
    FGSSL$^*$ & $10.29_{\pm 6.35}$ & $2.18_{\pm 1.45}$ & $36.96_{\pm 4.57}$ & $6.87_{\pm 3.61}$ & $15.11_{\pm 2.32}$ & $2.62_{\pm 2.01}$ & $20.30_{\pm 7.15}$ & $3.17_{\pm 1.36}$ & $84.48_{\pm 6.25}$ & $0.22_{\pm 0.20}$ \\
    FGGP$^*$ & $10.46_{\pm 3.75}$ & $2.06_{\pm 0.44}$ & $44.74_{\pm 2.38}$ & $6.91_{\pm 4.18}$ & $19.85_{\pm 3.68}$ & $5.42_{\pm 0.83}$ & $16.25_{\pm 1.14}$ & $3.78_{\pm 0.95}$ & $76.87_{\pm 3.33}$ & $0.31_{\pm 1.36}$ \\
    \rowcolor{gray!10}
    FedSPA$^*$ & $19.55_{\pm 5.13}$ & $5.49_{\pm 2.70}$ & $39.34_{\pm 4.91}$ & \textcolor{orange}{$\mathbf{8.12_{\pm 0.95}}$} & $11.17_{\pm 5.75}$ & $2.49_{\pm 2.37}$ & $18.53_{\pm 1.25}$ & $3.31_{\pm 0.07}$ & $75.75_{\pm 4.21}$ & $0.00_{\pm 0.06}$ \\
    S2FGL$^*$ & $11.26_{\pm 5.81}$ & $2.11_{\pm 2.45}$ & $34.96_{\pm 1.81}$ & $6.56_{\pm 1.51}$ & $9.61_{\pm 3.84}$ & $1.86_{\pm 2.42}$ & $10.25_{\pm 7.83}$ & $1.61_{\pm 0.58}$ & $63.19_{\pm 4.96}$ & $0.00_{\pm 0.08}$ \\
    \hline
    \rowcolor{gray!10}
    FedNCN & \textcolor{royalblue}{$\mathbf{57.83_{\pm 3.24}}$} & \textcolor{royalblue}{$\mathbf{16.91_{\pm 1.58}}$} & \textcolor{royalblue}{$\mathbf{63.72_{\pm 1.86}}$} & \textcolor{royalblue}{$\mathbf{10.15_{\pm 1.92}}$} & \textcolor{royalblue}{$\mathbf{69.94_{\pm 2.37}}$} & \textcolor{royalblue}{$\mathbf{20.68_{\pm 3.45}}$} & \textcolor{royalblue}{$\mathbf{72.15_{\pm 2.06}}$} & \textcolor{royalblue}{$\mathbf{28.43_{\pm 2.51}}$} & \textcolor{royalblue}{$\mathbf{90.58_{\pm 5.12}}$} & \textcolor{royalblue}{$\mathbf{1.52_{\pm 0.95}}$} \\
    AdaFGC (Ours) & \textcolor{darkred}{$\mathbf{65.09_{\pm 2.41}}$} & \textcolor{darkred}{$\mathbf{24.34_{\pm 1.19}}$} & \textcolor{darkred}{$\mathbf{70.02_{\pm 1.38}}$} & \textcolor{darkred}{$\mathbf{13.87_{\pm 1.69}}$} & \textcolor{darkred}{$\mathbf{79.70_{\pm 1.69}}$} & \textcolor{darkred}{$\mathbf{26.35_{\pm 2.84}}$} & \textcolor{darkred}{$\mathbf{77.11_{\pm 1.99}}$} & \textcolor{darkred}{$\mathbf{32.27_{\pm 2.25}}$} & \textcolor{darkred}{$\mathbf{94.31_{\pm 5.19}}$} & \textcolor{darkred}{$\mathbf{4.12_{\pm 0.86}}$} \\
    \hline\thickhline
    \rowcolor{gray!10}
    
    & \multicolumn{2}{c!{\vrule width 0.1pt}}{\textbf{CiteSeer}}  
    & \multicolumn{2}{c!{\vrule width 0.1pt}}{\textbf{PubMed}}  
    & \multicolumn{2}{c!{\vrule width 0.1pt}}{\textbf{Amazon-Computer}}  
    & \multicolumn{2}{c!{\vrule width 0.1pt}}{\textbf{Amazon-Photo}}  
    & \multicolumn{2}{c}{\textbf{Questions}}  \\
    
    \cline{2-11}
    \rowcolor{gray!10}
    \multirow{-2}{*}{\diagbox[width=8.5em,height=2.4em]{\textbf{Methods}}{\textbf{Datasets}}}  
    & ARI & F1 
    & ARI & F1 
    & ARI & F1 
    & ARI & F1 
    & ARI & F1 \\
    
    \hline
    FedSage+$^*$ & $0.00_{\pm 0.56}$ & $11.28_{\pm 0.60}$ & $0.03_{\pm 0.01}$ & $26.44_{\pm 5.49}$ & \textcolor{orange}{$\mathbf{6.32_{\pm 1.17}}$} & \textcolor{orange}{$\mathbf{17.44_{\pm 0.18}}$} & $9.15_{\pm 2.14}$ & \textcolor{orange}{$\mathbf{18.82_{\pm 0.59}}$} & $0.83_{\pm 2.00}$ & $39.23_{\pm 1.32}$ \\
    \rowcolor{gray!10}
    FedPUB$^*$ & $0.00_{\pm 0.01}$ & $2.97_{\pm 1.66}$ & $0.00_{\pm 0.01}$ & $13.73_{\pm 2.99}$ & $0.00_{\pm 0.01}$ & $4.87_{\pm 3.36}$ & $0.00_{\pm 0.01}$ & $1.85_{\pm 0.43}$ & $0.00_{\pm 0.01}$ & $31.24_{\pm 22.04}$ \\
    FedTAD$^*$ & $1.56_{\pm 1.21}$ & $9.84_{\pm 2.69}$ & $0.53_{\pm 0.74}$ & $19.86_{\pm 3.31}$ & $4.95_{\pm 2.50}$ & $3.55_{\pm 2.07}$ & $5.56_{\pm 2.61}$ & $4.14_{\pm 1.70}$ & $0.00_{\pm 0.37}$ & $39.89_{\pm 14.43}$ \\
    \rowcolor{gray!10}
    FedGTA$^*$ & \textcolor{orange}{$\mathbf{4.90_{\pm 2.24}}$} & \textcolor{orange}{$\mathbf{20.21_{\pm 0.63}}$} & $1.80_{\pm 1.01}$ & \textcolor{orange}{$\mathbf{34.20_{\pm 1.73}}$} & $3.78_{\pm 0.81}$ & $7.04_{\pm 0.53}$ & \textcolor{orange}{$\mathbf{9.80_{\pm 1.11}}$} & $10.97_{\pm 0.53}$ & \textcolor{orange}{$\mathbf{1.33_{\pm 0.42}}$} & \textcolor{orange}{$\mathbf{49.12_{\pm 0.46}}$} \\
    FedIIH$^*$ & $0.41_{\pm 0.54}$ & $5.00_{\pm 1.12}$ & $0.02_{\pm 0.03}$ & $14.71_{\pm 2.89}$ & $0.06_{\pm 0.09}$ & $1.79_{\pm 1.04}$ & $0.00_{\pm 0.01}$ & $2.65_{\pm 1.90}$ & $0.00_{\pm 0.01}$ & $30.36_{\pm 22.37}$ \\
    \rowcolor{gray!10}
    FGSSL$^*$ & $1.74_{\pm 0.82}$ & $7.77_{\pm 2.29}$ & $7.62_{\pm 0.34}$ & $27.18_{\pm 7.14}$ & $2.97_{\pm 0.01}$ & $10.11_{\pm 1.43}$ & $3.49_{\pm 0.04}$ & $14.26_{\pm 0.51}$ & $0.14_{\pm 0.27}$ & $33.67_{\pm 7.03}$ \\
    FGGP$^*$ & $2.29_{\pm 1.63}$ & $7.62_{\pm 1.09}$ & $5.10_{\pm 0.76}$ & $33.05_{\pm 5.68}$ & $2.38_{\pm 1.84}$ & $15.27_{\pm 3.35}$ & $3.35_{\pm 0.49}$ & $11.47_{\pm 0.50}$ & $1.25_{\pm 0.10}$ & $42.45_{\pm 9.94}$ \\
    \rowcolor{gray!10}
    FedSPA$^*$ & $3.99_{\pm 1.09}$ & $13.71_{\pm 1.28}$ & \textcolor{orange}{$\mathbf{9.09_{\pm 0.42}}$} & $27.61_{\pm 2.22}$ & $1.74_{\pm 1.21}$ & $7.96_{\pm 0.65}$ & $3.53_{\pm 2.88}$ & $11.68_{\pm 1.94}$ & $0.25_{\pm 0.61}$ & $33.59_{\pm 17.84}$ \\
    S2FGL$^*$ & $2.32_{\pm 0.82}$ & $8.21_{\pm 1.40}$ & $6.46_{\pm 0.33}$ & $21.77_{\pm 4.59}$ & $1.81_{\pm 1.81}$ & $7.00_{\pm 1.78}$ & $1.68_{\pm 0.44}$ & $6.39_{\pm 1.76}$ & $0.36_{\pm 0.47}$ & $45.21_{\pm 17.31}$ \\
    \hline
    \rowcolor{gray!10}
    FedNCN & \textcolor{royalblue}{$\mathbf{14.89_{\pm 4.38}}$} & \textcolor{royalblue}{$\mathbf{26.35_{\pm 3.17}}$} & \textcolor{royalblue}{$\mathbf{12.84_{\pm 4.89}}$} & \textcolor{royalblue}{$\mathbf{43.21_{\pm 3.78}}$} & \textcolor{royalblue}{$\mathbf{22.58_{\pm 4.92}}$} & \textcolor{royalblue}{$\mathbf{23.14_{\pm 0.96}}$} & \textcolor{royalblue}{$\mathbf{29.87_{\pm 3.76}}$} & \textcolor{royalblue}{$\mathbf{26.73_{\pm 1.34}}$} & \textcolor{royalblue}{$\mathbf{4.31_{\pm 2.58}}$} & \textcolor{royalblue}{$\mathbf{50.28_{\pm 2.43}}$} \\
    AdaFGC (Ours) & \textcolor{darkred}{$\mathbf{19.27_{\pm 3.64}}$} & \textcolor{darkred}{$\mathbf{37.85_{\pm 2.65}}$} & \textcolor{darkred}{$\mathbf{16.43_{\pm 4.41}}$} & \textcolor{darkred}{$\mathbf{49.50_{\pm 3.34}}$} & \textcolor{darkred}{$\mathbf{27.44_{\pm 4.15}}$} & \textcolor{darkred}{$\mathbf{30.86_{\pm 0.74}}$} & \textcolor{darkred}{$\mathbf{33.72_{\pm 3.63}}$} & \textcolor{darkred}{$\mathbf{30.24_{\pm 1.13}}$} & \textcolor{darkred}{$\mathbf{6.87_{\pm 2.47}}$} & \textcolor{darkred}{$\mathbf{54.05_{\pm 2.27}}$} \\
    \hline\thickhline
    \end{tabular}
    }}
    \vspace{+0.2cm}
    \end{table*}

\subsection{Main Results}

To answer \textbf{Q1}, we report the clustering performances on five graph datasets in Table~\ref{tab: main}. We further evaluate these methods with varying numbers of participating clients in Appendix~\ref{appendix: more experiments}.

\vspace{+0.1cm}
\noindent \textbf{Comparison with Supervised FGL Methods.}
AdaFGC consistently achieves the best performance across all datasets and evaluation metrics, demonstrating stable superiority over existing supervised FGL methods. For instance, on CiteSeer, AdaFGC surpasses the strongest supervised baseline FedGTA by +37.93\% in ACC and +15.92\% in NMI. In contrast, many supervised FGL methods suffer severe performance degradation when adapted to clustering with pseudo labels, with several approaches (e.g., FedPUB and FedIIH) even producing near-zero NMI scores, indicating their failure to capture meaningful cluster structures.
    \begin{figure}
    \centering
    \includegraphics[width=0.48\textwidth]{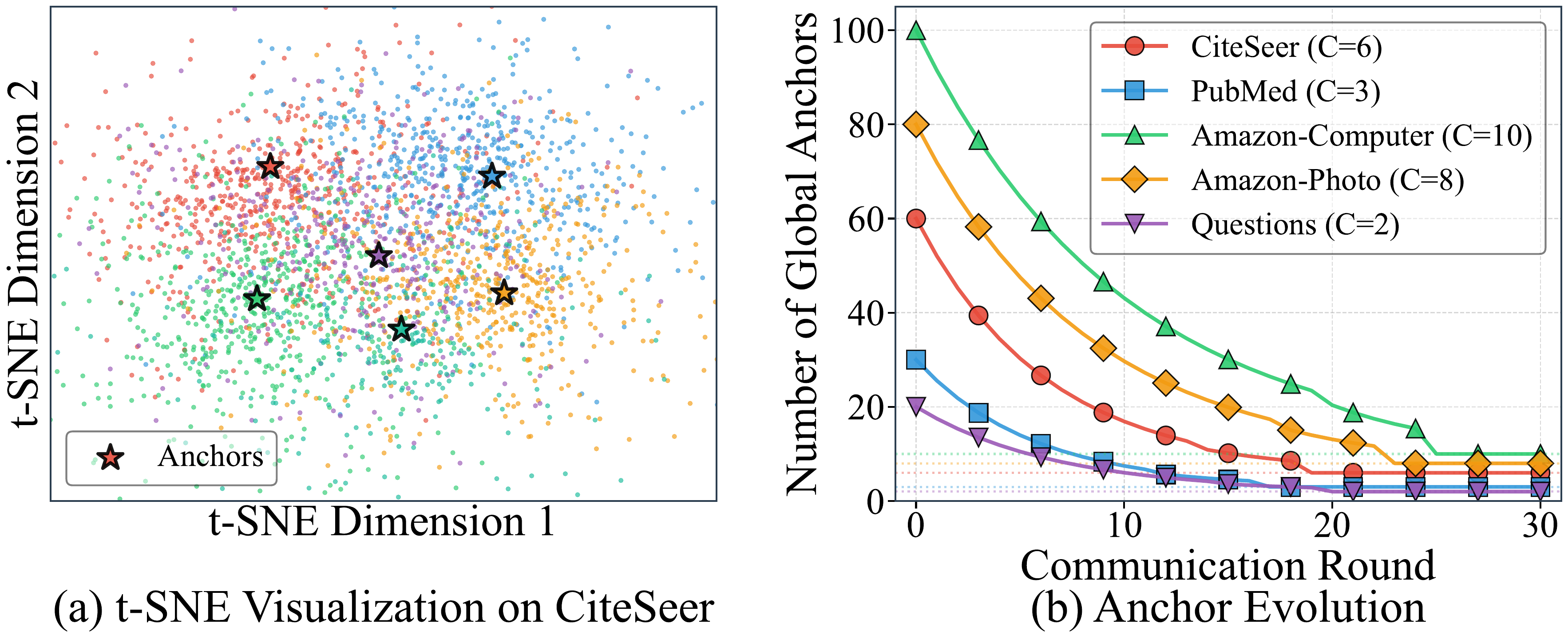}
    \caption{\textbf{Interpretability investigation.} (a) t-SNE visualization of node embeddings from all clients with global anchors. (b) Evolution of global anchor count during training, starting from $\hat{C}=10C$ and converging to the true number of communities $C$ for each dataset.}
    \label{fig: interpretability}
    \end{figure}

\vspace{+0.1cm}
\noindent \textbf{Comparison with Unsupervised FGL Methods.}
AdaFGC also consistently outperforms the state-of-the-art unsupervised FGL method FedNCN across all datasets and metrics. For instance, on CiteSeer, AdaFGC improves ACC, NMI, ARI, and F1 by +7.26\%, +7.43\%, +4.38\%, and +11.50\%, respectively. This demonstrates that the proposed global community-aware contrastive learning effectively captures cross-client community structures, while the adaptive anchor evolution mechanism better addresses the heterogeneity among federated subgraphs.

\subsection{Interpretability Investigation}

To answer \textbf{Q2}, we investigate how AdaFGC reveals meaningful community structures across federated clients through embedding visualization and global community anchor evolution tracking.

\vspace{+0.1cm}
\noindent \textbf{Embedding Visualization.}
Fig.~\ref{fig: interpretability} (a) presents a t-SNE visualization of node embeddings from all 10 clients on the CiteSeer dataset after training convergence. Each point represents a node, colored by its ground-truth community label, while star markers indicate the learned global community anchors. The visualization demonstrates that: (1) nodes from the same community exhibit discernible clustering tendencies despite originating from heterogeneous clients, validating that AdaFGC captures meaningful global community structures to address \textbf{L2} via the global community-aware contrastive learning mechanism; (2) the global anchors are consistently positioned within or near the high-density regions of their respective communities, confirming their effectiveness as local clustering guidance even under non-trivial data heterogeneity.

\vspace{+0.1cm}
\noindent \textbf{Anchor Evolution Tracking.}
Fig.~\ref{fig: interpretability} (b) tracks the number of global anchors throughout training across all five datasets. We initialize the anchor count as $\hat{C} = 10C$, where $C$ is the true number of communities for each dataset. For instance, CiteSeer with $C=6$ starts with 60 anchors, while Questions with $C=2$ starts with 20 anchors. The anchor count rapidly decreases in early rounds as redundant anchors merge based on multi-client clustering signals. By round 40-60, the anchor count stabilizes near the true number of communities for each dataset, demonstrating that AdaFGC effectively resolves \textbf{L1} without requiring pre-defined cluster cardinality. The convergence patterns are consistent across datasets with varying characteristics (ranging from $C=2$ to $C=10$), highlighting the effectiveness of the adaptive anchor evolution mechanism. In addition, the smooth stabilization process indicates that anchor refinement is progressive, verifying the robustness of the adaptive evolution strategy.

\subsection{Ablation Study}

\begin{table}[htbp]
    \setlength{\abovecaptionskip}{0.2cm}
    \setlength{\belowcaptionskip}{-0.2cm}
    \centering
    \caption{\textbf{Ablation study} on three datasets. We evaluate the contribution of each loss in AdaFGC.}
    \label{tab: ablation}
    \footnotesize 
    \renewcommand{\arraystretch}{1.1}
    \resizebox{\linewidth}{!}{
    \setlength{\tabcolsep}{2.5mm}{
    \begin{tabular}{c!{\vrule width 0.1pt}
    cc!{\vrule width 0.1pt}
    cc!{\vrule width 0.1pt}
    cc}
    \hline\thickhline
    \rowcolor{gray!10}
    
    & \multicolumn{2}{c!{\vrule width 0.1pt}}{\textbf{Amazon-Computer}}  
    & \multicolumn{2}{c!{\vrule width 0.1pt}}{\textbf{Amazon-Photo}}  
    & \multicolumn{2}{c}{\textbf{Questions}}  \\
    
    \cline{2-7}
    \rowcolor{gray!10}
    \multirow{-2}{*}{\diagbox[width=10em,height=2.4em]{\textbf{Variants}}{\textbf{Datasets}}}  
    & NMI & ARI 
    & NMI & ARI 
    & NMI & ARI \\
    
    \hline
    \rowcolor{gray!10}
    w/o. $\mathcal{L}_\mathrm{anc}$ & \textcolor{orange}{$\mathbf{24.27_{\pm 2.67}}$} & \textcolor{orange}{$\mathbf{25.21_{\pm 3.98}}$} & \textcolor{orange}{$\mathbf{30.05_{\pm 2.18}}$} & \textcolor{orange}{$\mathbf{31.48_{\pm 3.51}}$} & \textcolor{orange}{$\mathbf{3.51_{\pm 0.79}}$} & \textcolor{orange}{$\mathbf{5.24_{\pm 2.35}}$} \\
    w/o. $\mathcal{L}_\mathrm{node}$ & $23.47_{\pm 2.71}$ & $24.72_{\pm 4.03}$ & $27.38_{\pm 2.21}$ & $29.79_{\pm 3.47}$ & $3.28_{\pm 0.81}$ & $4.97_{\pm 2.39}$ \\
    \rowcolor{gray!10}
    w/o. $\mathcal{L}_\mathrm{topo}$ & \textcolor{royalblue}{$\mathbf{24.68_{\pm 2.79}}$} & \textcolor{royalblue}{$\mathbf{25.82_{\pm 4.08}}$} & \textcolor{royalblue}{$\mathbf{31.61_{\pm 2.19}}$} & \textcolor{royalblue}{$\mathbf{32.15_{\pm 3.56}}$} & \textcolor{royalblue}{$\mathbf{3.74_{\pm 0.84}}$} & \textcolor{royalblue}{$\mathbf{5.51_{\pm 2.41}}$} \\
    \hline
    \rowcolor{gray!10}
    AdaFGC (Full) & \textcolor{darkred}{$\mathbf{26.35_{\pm 2.84}}$} & \textcolor{darkred}{$\mathbf{27.44_{\pm 4.15}}$} & \textcolor{darkred}{$\mathbf{32.27_{\pm 2.25}}$} & \textcolor{darkred}{$\mathbf{33.72_{\pm 3.63}}$} & \textcolor{darkred}{$\mathbf{4.12_{\pm 0.86}}$} & \textcolor{darkred}{$\mathbf{6.87_{\pm 2.47}}$} \\
    \hline\thickhline
    \end{tabular}
    }}
\end{table}

To address \textbf{Q3}, we evaluate the contribution of each loss component by removing: (1) $\mathcal{L}_\mathrm{anc}$ (anchor-based community-level contrastive loss), (2) $\mathcal{L}_\mathrm{node}$ (node-level contrastive loss), and (3) $\mathcal{L}_\mathrm{topo}$ (topology-level contrastive loss). Other core components (e.g., global community-aware clustering and anchor evolution) are retained as they are essential for model functionality.

As shown in Table~\ref{tab: ablation}, removing any loss degrades performance. Excluding $\mathcal{L}_\mathrm{node}$ causes the largest drop (e.g., NMI/ARI decrease by 2.88/2.72 on Amazon-Computer), confirming its critical role in node representation learning. Removing $\mathcal{L}_\mathrm{anc}$ yields the second-largest decline, highlighting the importance of global community guidance. These results validate that all three losses jointly contribute to AdaFGC's effectiveness.

\begin{figure*}[htb]
 \centering
  \includegraphics[width=0.90\textwidth]{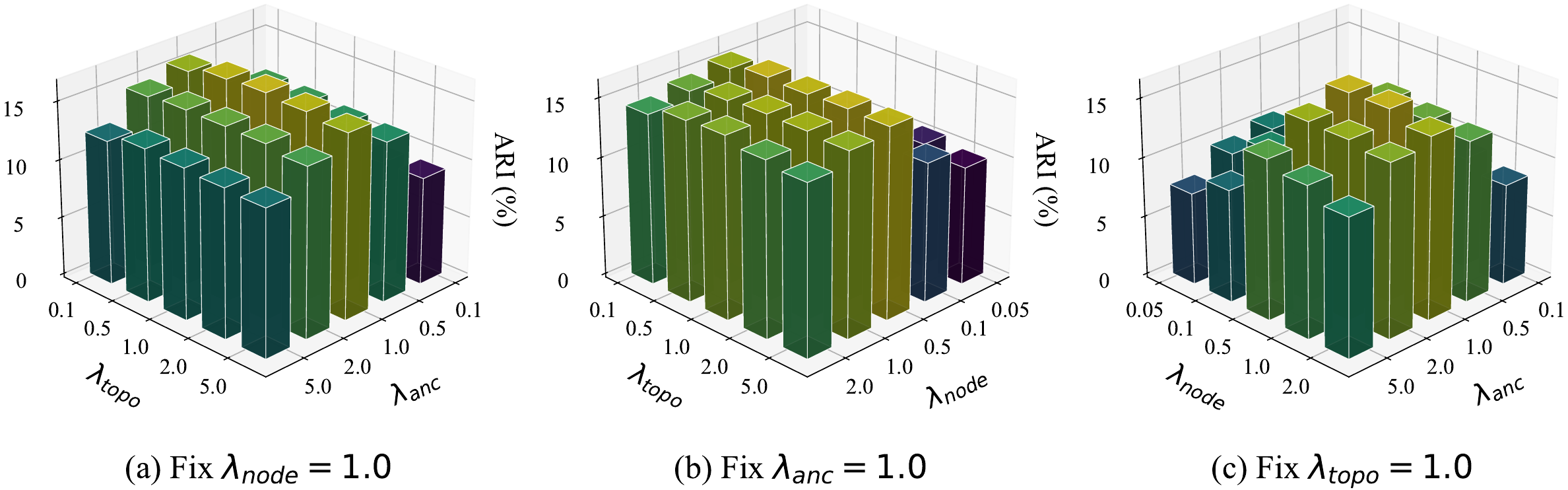}
\caption{\textbf{Hyperparameter analysis for $\lambda_\mathrm{anc}$, $\lambda_\mathrm{node}$, $\lambda_\mathrm{topo}$} on the PubMed dataset.}
\label{fig: robust}
\end{figure*}

\subsection{Hyperparameter Sensitivity}

To address \textbf{Q4}, we investigate the sensitivity of AdaFGC to the loss weight hyperparameters $\lambda_\mathrm{anc}$, $\lambda_\mathrm{node}$, and $\lambda_\mathrm{topo}$ on the PubMed dataset, as shown in Fig.~\ref{fig: robust}. We observe that AdaFGC is primarily sensitive to $\mathcal{L}_\mathrm{node}$ (controlled by $\lambda_\mathrm{node}$) and $\mathcal{L}_\mathrm{anc}$ (controlled by $\lambda_\mathrm{anc}$). When their weights are too low, AdaFGC exhibits relatively poor performance, as insufficient node-level representation learning and weak global community guidance fail to capture meaningful clustering structures. However, when adjusted within a reasonable range, AdaFGC maintains overall stable performance across different hyperparameter configurations, demonstrating its robustness to hyperparameter choices in practical federated scenarios.

\subsection{Efficiency Analysis}
\label{sec: efficiency}

    \begin{figure}
    \centering
    \includegraphics[width=0.30\textwidth]{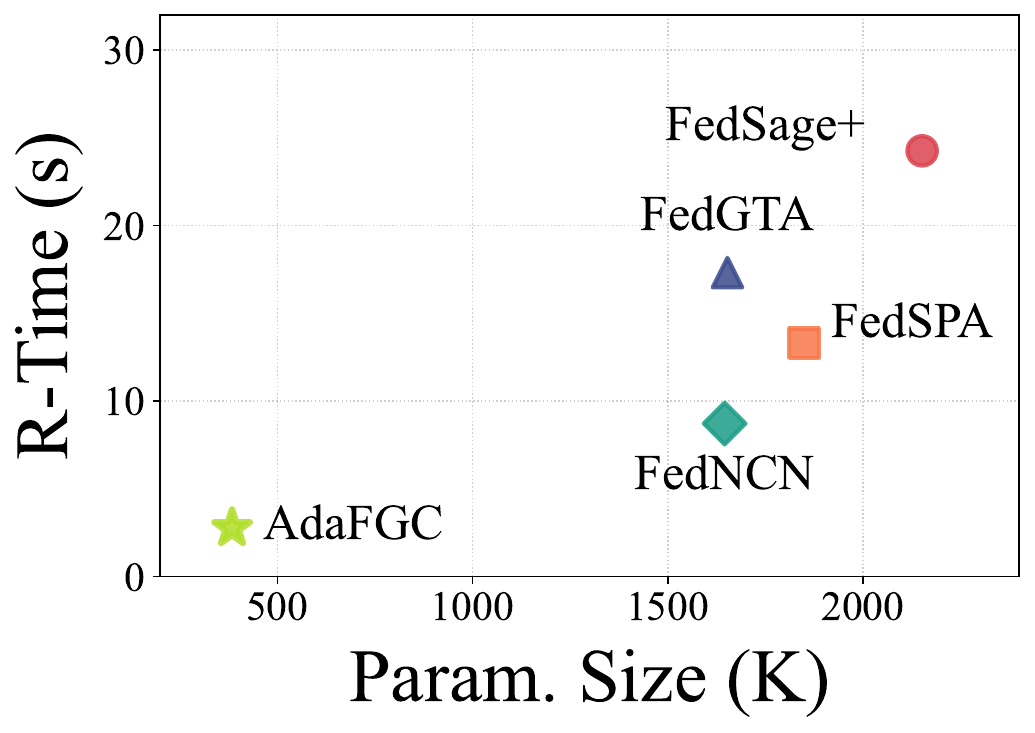}
    \caption{\textbf{Efficiency analysis} on Amazon-Computers dataset.}
    \label{fig: efficiency}
    \end{figure}

\begin{figure*}[htb]
 \centering
  \includegraphics[width=0.998\textwidth]{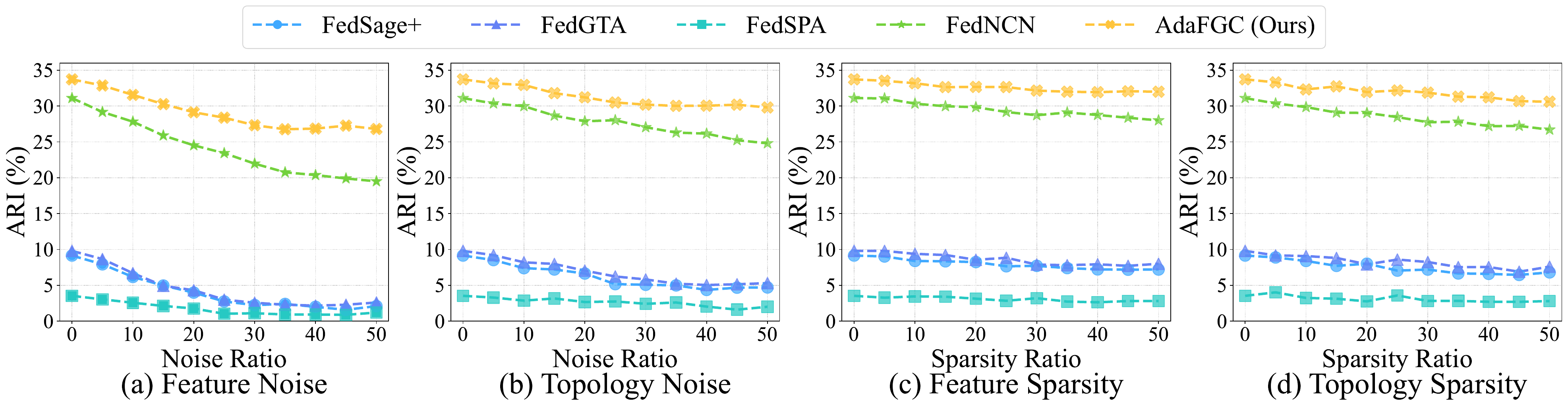}
\caption{\textbf{Robustness analysis} on the Amazon-Photo dataset. We investigate four types of perturbations: (a) \textbf{feature noise} (i.e., adding Gaussian noise to node features), (b) \textbf{topology noise} (i.e., adding random edges to the graph), (c) \textbf{feature sparsity} (i.e., masking node features), and (d) \textbf{topology sparsity} (i.e., removing edges from the graph).}
\label{fig: robust}
\vspace{+0.3cm}
\end{figure*}

To answer \textbf{Q5}, we evaluate the efficiency of AdaFGC in terms of parameter size and computational cost. As shown in Table~\ref{fig: efficiency}, AdaFGC consistently outperforms all baselines on the Amazon-Computer dataset. Specifically, AdaFGC requires only 384.00 K parameters, representing a reduction of approximately 76.7\% compared to the most competitive baseline FedNCN (1645.60 K) and 82.2\% compared to FedSage+. In terms of computational speed, AdaFGC achieves a running time of 2.77 s per round, which is over 3.1× faster than FedNCN (8.72 s) and nearly 8.7× faster than FedSage+ (24.26 s). These results demonstrate that AdaFGC not only achieves superior clustering performance but also significantly reduces the communication and computation burden, making it highly suitable for resource-constrained federated environments.

\begin{figure}
    \centering
    \includegraphics[width=0.38\textwidth]{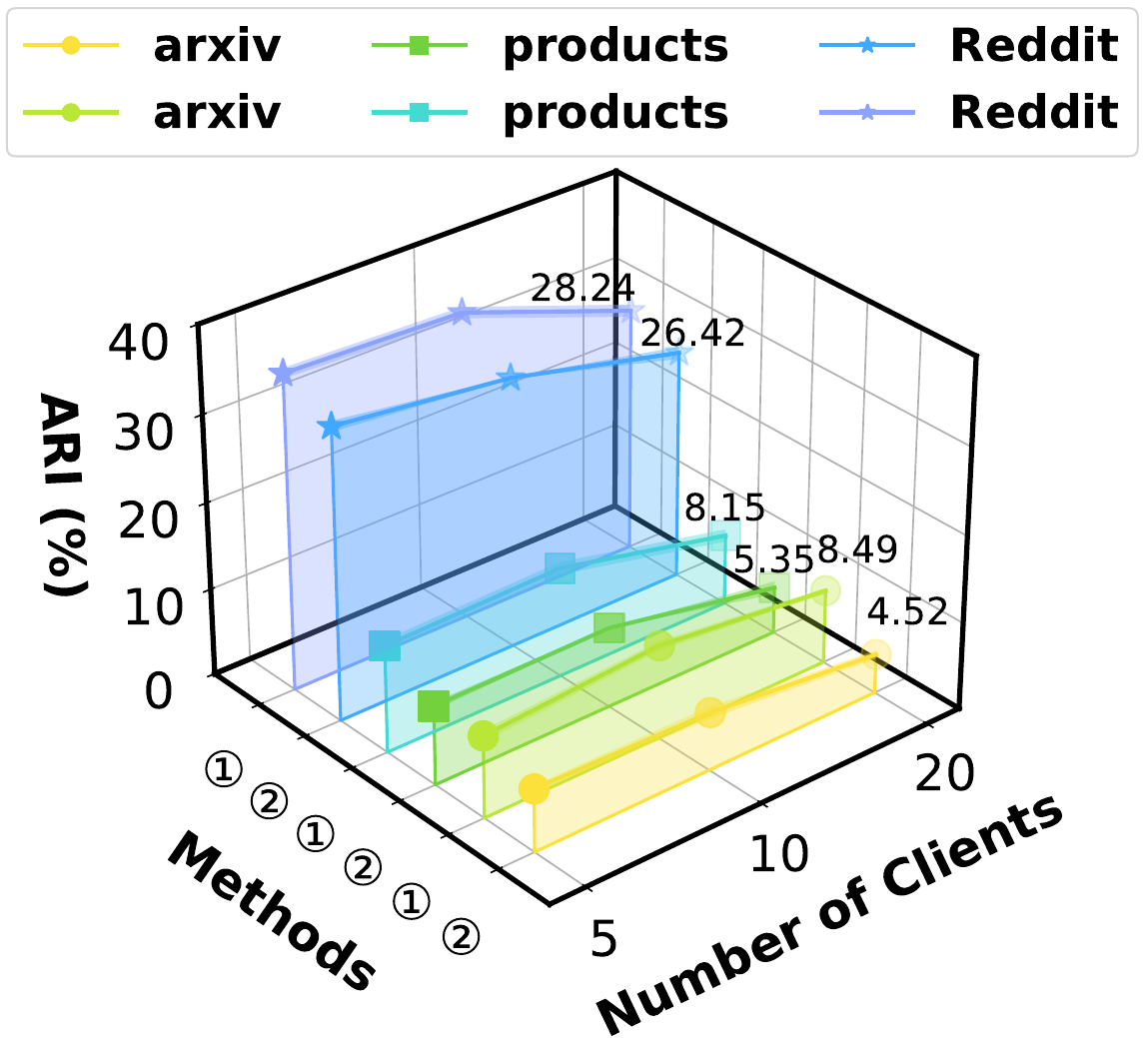}
    \caption{\textbf{Clustering results on large-scale graphs}, where `\ding{172}' denotes FedNCN and `\ding{173}' denotes AdaFGC (Ours).}
    \label{fig: large_scale}
\end{figure}

Moreover, we further evaluate the scalability of AdaFGC on three large-scale datasets: ogb-arxiv (169K nodes), ogb-products (2.45M nodes), and Reddit (233K nodes). We compare AdaFGC with FedNCN, the strongest baseline from the main experiments, under different numbers of clients ($K \in \{5,10,20\}$). Each client receives a subgraph generated by Louvain-based partitioning to simulate community-aware distributed data. As shown in Fig.~\ref{fig: large_scale}, AdaFGC consistently outperforms FedNCN across all datasets and client settings in terms of ARI. Although clustering performance generally decreases as the number of clients increases due to stronger data fragmentation, AdaFGC maintains a clear advantage, especially on larger graphs such as ogb-products. These results demonstrate that AdaFGC can effectively scale to large graphs while preserving superior clustering quality in federated environments.

\subsection{Robustness Analysis}

To answer \textbf{Q6}, we evaluate the robustness of AdaFGC against four types of perturbations (i.e., feature noise, topology noise, feature sparsity, and topology sparsity) as illustrated in Fig. \ref{fig: robust}. The results clearly show that supervised FGL methods (FedSage+, FedGTA, and FedSPA) consistently struggle, with ARI scores remaining below 10\% across all scenarios; this confirms that relying on pseudo-labels in federated settings fails to capture stable community structures, especially under severe data corruption. While the unsupervised baseline FedNCN proves to be a stronger competitor, it still lags significantly behind AdaFGC. Notably, AdaFGC at 50\% feature sparsity outperforms even the unperturbed (0\%) FedNCN. This resilience is mainly driven by our global community-aware contrastive learning, which effectively compensates for local sparsity through global semantic consistency. Overall, AdaFGC exhibits strong robustness to imperfect real-world graph data.

\section{Theoretical Proofs}
\label{appendix: proofs}

\begin{theorem}[Graph Filtering as Low-pass Smoothing]
\label{thm: smooth fused rep}
For client $k$ with feature-view embeddings $\mathbf{Z}^k$ and symmetrically normalized adjacency matrix $\tilde{A}$, the topology-view embeddings $\mathbf{H}^k$ obtained via graph filtering can be expressed as:
\begin{equation}
    \mathbf{H}^k = \frac{1}{\alpha+1}\sum_{l=0}^L\Big(\frac{\alpha}{\alpha+1} \tilde{A}\Big)^l \mathbf{Z}^k,
\end{equation}
where $\alpha>0$ controls smoothing strength and $L$ is the propagation depth.
\end{theorem}

\begin{proof} 
The topology-view embeddings that maintain smoothness over the graph structure are obtained by minimizing:
\begin{equation}
    \label{proof_eq_smooth}
    {\mathbf{H}^k} = \argmin_{\mathbf{H}^k} f(\mathbf{H}^k)=\argmin_{\mathbf{H}^k} \|\mathbf{H}^k - \mathbf{Z}^k\|_F^2 + \alpha \cdot \text{tr}\big({\mathbf{H}^k}^\top(\mathbf{I}-\tilde{A})\mathbf{H}^k\big),
\end{equation}
where $\alpha > 0$ controls the smoothness strength, and the Laplacian regularizer $\mathbf{I}-\tilde{A}$ encourages neighboring nodes to have similar embeddings. The fidelity term $\|\mathbf{H}^k - \mathbf{Z}^k\|_F^2$ ensures that the learned embeddings remain close to the initial feature-view embeddings.

Setting $\frac{\partial f(\mathbf{H}^k)}{\partial \mathbf{H}^k}=0$ yields:
\begin{equation}
    {\mathbf{H}^k} = \frac{1}{\alpha+1} \Big(\mathbf{I}-\frac{\alpha}{\alpha+1} \tilde{A}\Big)^{-1}\mathbf{Z}^k.
\end{equation}

Note that $\tilde{A}$ is the symmetrically normalized adjacency matrix, which is symmetric with real eigenvalues bounded by 1. The maximum eigenvalue of a symmetric matrix can be characterized by the Rayleigh quotient:
\[
\lambda_{\max}(\tilde{A}) = \max_{x \neq 0} \frac{x^\top \tilde{A} x}{x^\top x}.
\]

For the symmetrically normalized adjacency matrix $\tilde{A} = D^{-1/2}AD^{-1/2}$, where $D$ is the degree matrix and $A$ is the adjacency matrix, the eigenvalues are bounded:
\[
\lambda_{\max}(\tilde{A}) \le 1.
\]

Thus, the dominant eigenvalue of $\frac{\alpha}{\alpha+1}\tilde{A}$ is guaranteed to be $< 1$ for any $\alpha > 0$, ensuring stability in the graph filtering operation. As shown in~\cite{matrix_analysis}, when the dominant eigenvalue is less than 1, the inverse $\Big(\mathbf{I}-\frac{\alpha}{\alpha+1}\tilde{A}\Big)^{-1}$ can be expressed as a Neumann series:
\begin{equation}
    \Big(\mathbf{I}-\frac{\alpha}{\alpha+1}\tilde{A}\Big)^{-1} = \sum_{l=0}^\infty \Big(\frac{\alpha}{\alpha+1}\tilde{A}\Big)^{l}.
\end{equation}

Truncating the series at $L$ steps gives the power series approximation:
\begin{equation}
    \mathbf{H}^k
    =
    \frac{1}{\alpha+1}
    \sum_{l=0}^{L}
    \Big(\frac{\alpha}{\alpha+1} \tilde{A}\Big)^l
    \mathbf{Z}^k,
\end{equation}
which matches the graph filtering formulation in Eq.~\eqref{eq: forward}. This demonstrates that many widely used GNNs (e.g., GCN, SGC) can be viewed as special cases of this low-pass graph filtering operation.
\end{proof}

\section{Conclusion}
\label{sec: Conclusion}

In this paper, we present AdaFGC, a novel framework that advances federated graph clustering by addressing the limitations of pre-defined cluster cardinality and incomplete inter-community separation. By introducing an over-complete set of global community anchors alongside an adaptive refinement mechanism, AdaFGC effectively bypasses the need for prior knowledge of cluster numbers, allowing the global community structure to emerge dynamically and progressively from distributed subgraphs. Through the integration of global community-aware contrastive learning, our approach leverages shared anchors as prototypes to enforce clearer discriminative boundaries and enhance representation stability across clients via community, node, and topology-level objectives. Extensive empirical evaluations on eight benchmark datasets demonstrate that AdaFGC consistently outperforms existing FGL baselines, offering an effective solution for distilling knowledge from real-world distributed, unlabeled graph repositories.

\newpage
\bibliographystyle{ACM-Reference-Format}
\balance
\bibliography{citation}

@inproceedings{Yang16cora, 
    author = {Yang, Zhilin and Cohen, William W. and Salakhutdinov, Ruslan}, title = {Revisiting Semi-Supervised Learning with Graph Embeddings}, 
    year = {2016}, 
    booktitle={International Conference on Machine Learning, ICML}, 
}

@article{shchur2018amazon_datasets,
  title={Pitfalls of graph neural network evaluation},
  author={Shchur, Oleksandr and Mumme, Maximilian and Bojchevski, Aleksandar and G{\"u}nnemann, Stephan},
  journal={arXiv preprint arXiv:1811.05868},
  year={2018}
}

@article{platonov2023hete_gnn_survey4,
  title={A critical look at the evaluation of GNNs under heterophily: are we really making progress?},
  author={Platonov, Oleg and Kuznedelev, Denis and Diskin, Michael and Babenko, Artem and Prokhorenkova, Liudmila},
  journal = {International Conference on Learning Representations, ICLR},
  year={2023}
}

@article{louvain,
  title={Fast unfolding of communities in large networks},
  author={Blondel, Vincent D and Guillaume, Jean-Loup and Lambiotte, Renaud and Lefebvre, Etienne},
  journal={Journal of Statistical Mechanics: Theory and Experiment},
  volume={2008},
  number={10},
  pages={P10008},
  year={2008},
  publisher={IOP Publishing}
}

@inproceedings{fedavg,
  title={Communication-efficient learning of deep networks from decentralized data},
  author={McMahan, Brendan and Moore, Eider and Ramage, Daniel and Hampson, Seth and y Arcas, Blaise Aguera},
  booktitle={Artificial intelligence and statistics},
  pages={1273--1282},
  year={2017},
  organization={Pmlr}
}

@article{fedprox,
  title={Federated optimization in heterogeneous networks},
  author={Li, Tian and Sahu, Anit Kumar and Zaheer, Manzil and Sanjabi, Maziar and Talwalkar, Ameet and Smith, Virginia},
  journal={Proceedings of Machine learning and systems},
  volume={2},
  pages={429--450},
  year={2020}
}

@article{fedtad,
  title={Fedtad: Topology-aware data-free knowledge distillation for subgraph federated learning},
  author={Zhu, Yinlin and Li, Xunkai and Wu, Zhengyu and Wu, Di and Hu, Miao and Li, Rong-Hua},
  journal={arXiv preprint arXiv:2404.14061},
  year={2024}
}

@inproceedings{fedpub,
  title={Personalized subgraph federated learning},
  author={Baek, Jinheon and Jeong, Wonyong and Jin, Jiongdao and Yoon, Jaehong and Hwang, Sung Ju},
  booktitle={International conference on machine learning},
  pages={1396--1415},
  year={2023},
  organization={PMLR}
}

@article{fedsage_plus,
  title={Subgraph federated learning with missing neighbor generation},
  author={Zhang, Ke and Yang, Carl and Li, Xiaoxiao and Sun, Lichao and Yiu, Siu Ming},
  journal={Advances in neural information processing systems},
  volume={34},
  pages={6671--6682},
  year={2021}
}

@article{fedgta,
  title={FedGTA: Topology-Aware Averaging for Federated Graph Learning},
  author={Li, Xunkai and Wu, Zhengyu and Zhang, Wentao and Zhu, Yinlin and Li, Rong-Hua and Wang, Guoren},
  journal={Proceedings of the VLDB Endowment},
  year={2023},
  publisher={VLDB Endowment}
}

@inproceedings{fedgcn,
  title={Federated graph-level clustering network},
  author={Liu, Jingxin and Cheng, Jieren and Han, Renda and Tu, Wenxuan and Wang, Jiaxin and Peng, Xin},
  booktitle={Proceedings of the AAAI Conference on Artificial Intelligence},
  volume={39},
  number={18},
  pages={18870--18878},
  year={2025}
}

@inproceedings{fedncn,
  title={Federated node-level clustering network with cross-subgraph link mending},
  author={Liu, Jingxin and Han, Renda and Tu, Wenxuan and Wang, Haotian and Wu, Junlong and Cheng, Jieren},
  booktitle={Forty-second International Conference on Machine Learning},
  year={2025}
}

@misc{fgssl,
      title={Federated Graph Semantic and Structural Learning}, 
      author={Wenke Huang and Guancheng Wan and Mang Ye and Bo Du},
      year={2024},
      eprint={2406.18937},
      archivePrefix={arXiv},
      primaryClass={cs.LG},
      url={https://arxiv.org/abs/2406.18937}, 
}

@inproceedings{fggp,
  title={Federated graph learning under domain shift with generalizable prototypes},
  author={Wan, Guancheng and Huang, Wenke and Ye, Mang},
  booktitle={Proceedings of the AAAI conference on artificial intelligence},
  volume={38},
  number={14},
  pages={15429--15437},
  year={2024}
}

@inproceedings{fedstar,
  title={Federated learning on non-iid graphs via structural knowledge sharing},
  author={Tan, Yue and Liu, Yixin and Long, Guodong and Jiang, Jing and Lu, Qinghua and Zhang, Chengqi},
  booktitle={Proceedings of the AAAI conference on artificial intelligence},
  volume={37},
  number={8},
  pages={9953--9961},
  year={2023}
}

@inproceedings{fediih,
  title={Modeling inter-intra heterogeneity for graph federated learning},
  author={Yu, Wentao and Chen, Shuo and Tong, Yongxin and Gu, Tianlong and Gong, Chen},
  booktitle={Proceedings of the AAAI Conference on Artificial Intelligence},
  volume={39},
  number={21},
  pages={22236--22244},
  year={2025}
}

@article{fedspa,
  title={FedSPA : Generalizable Federated Graph Learning under Homophily Heterogeneity},
  author={Zihan Tan and Guancheng Wan and Wenke Huang and He Li and Guibin Zhang and Carl Yang and Mang Ye},
  journal={2025 IEEE/CVF Conference on Computer Vision and Pattern Recognition (CVPR)},
  year={2025},
  pages={15464-15475},
  url={https://api.semanticscholar.org/CorpusID:277057037}
}

@article{s2fgl,
  title={S2FGL: Spatial Spectral Federated Graph Learning},
  author={Tan, Zihan and Huang, Suyuan and Wan, Guancheng and Huang, Wenke and Li, He and Ye, Mang},
  journal={arXiv preprint arXiv:2507.02409},
  year={2025}
}

@article{gcfl_plus,
  title={Federated graph classification over non-iid graphs},
  author={Xie, Han and Ma, Jing and Xiong, Li and Yang, Carl},
  journal={Advances in neural information processing systems},
  volume={34},
  pages={18839--18852},
  year={2021}
}

@inproceedings{fedpka,
  title={FedPKA: Federated Graph-Level Clustering Network with Personalized Knowledge Aggregation},
  author={Wu, Junlong and Wang, Haotian and Liu, Jingxin and Tu, Wenxuan and Han, Renda and Cheng, Jieren and Tang, Xiangyan},
  booktitle={International Conference on Intelligent Computing},
  pages={27--38},
  year={2025},
  organization={Springer}
}

@article{fgl_survey,
  title={Federated graph neural networks: Overview, techniques, and challenges},
  author={Liu, Rui and Xing, Pengwei and Deng, Zichao and Li, Anran and Guan, Cuntai and Yu, Han},
  journal={IEEE transactions on neural networks and learning systems},
  year={2024},
  publisher={IEEE}
}

@article{data_centric_fgl_survey,
  title={A comprehensive data-centric overview of federated graph learning},
  author={Wu, Zhengyu and Li, Xunkai and Zhu, Yinlin and Chen, Zekai and Yan, Guochen and Yan, Yanyu and Zhang, Hao and Ai, Yuming and Jin, Xinmo and Li, Rong-Hua and others},
  journal={arXiv preprint arXiv:2507.16541},
  year={2025}
}

@article{openfgl,
  title={Openfgl: A comprehensive benchmark for federated graph learning},
  author={Li, Xunkai and Zhu, Yinlin and Pang, Boyang and Yan, Guochen and Yan, Yeyu and Li, Zening and Wu, Zhengyu and Zhang, Wentao and Li, Rong-Hua and Wang, Guoren},
  journal={arXiv preprint arXiv:2408.16288},
  year={2024}
}

@misc{fedgfm,
      title={Towards Effective Federated Graph Foundation Model via Mitigating Knowledge Entanglement}, 
      author={Yinlin Zhu and Xunkai Li and Jishuo Jia and Miao Hu and Di Wu and Meikang Qiu},
      year={2025},
      eprint={2505.12684},
      archivePrefix={arXiv},
      primaryClass={cs.LG},
      url={https://arxiv.org/abs/2505.12684}, 
}

@article{gpt4,
  title={Gpt-4 technical report},
  author={Achiam, Josh and Adler, Steven and Agarwal, Sandhini and Ahmad, Lama and Akkaya, Ilge and Aleman, Florencia Leoni and Almeida, Diogo and Altenschmidt, Janko and Altman, Sam and Anadkat, Shyamal and others},
  journal={arXiv preprint arXiv:2303.08774},
  year={2023}
}

@article{sora,
  title={Sora: A review on background, technology, limitations, and opportunities of large vision models},
  author={Liu, Yixin and Zhang, Kai and Li, Yuan and Yan, Zhiling and Gao, Chujie and Chen, Ruoxi and Yuan, Zhengqing and Huang, Yue and Sun, Hanchi and Gao, Jianfeng and others},
  journal={arXiv preprint arXiv:2402.17177},
  year={2024}
}

@article{fedssp,
  title={FedSSP: federated graph learning with spectral knowledge and personalized preference},
  author={Tan, Zihan and Wan, Guancheng and Huang, Wenke and Ye, Mang},
  journal={Advances in Neural Information Processing Systems},
  volume={37},
  pages={34561--34581},
  year={2024}
}

@inproceedings{power,
  title={Federated continual graph learning},
  author={Zhu, Yinlin and Hu, Miao and Wu, Di},
  booktitle={Proceedings of the 31st ACM SIGKDD Conference on Knowledge Discovery and Data Mining V. 2},
  pages={4203--4213},
  year={2025}
}

@inproceedings{liu2022efficient,
  title={Efficient one-pass multi-view subspace clustering with consensus anchors},
  author={Liu, Suyuan and Wang, Siwei and Zhang, Pei and Xu, Kai and Liu, Xinwang and Zhang, Changwang and Gao, Feng},
  booktitle={Proceedings of the AAAI conference on artificial intelligence},
  volume={36},
  number={7},
  pages={7576--7584},
  year={2022}
}

@article{li2022high,
  title={High-order correlation preserved incomplete multi-view subspace clustering},
  author={Li, Zhenglai and Tang, Chang and Zheng, Xiao and Liu, Xinwang and Zhang, Wei and Zhu, En},
  journal={IEEE Transactions on Image Processing},
  volume={31},
  pages={2067--2080},
  year={2022},
  publisher={IEEE}
}

@article{gong2022deep,
  title={Deep fusion clustering network with reliable structure preservation},
  author={Gong, Lei and Tu, Wenxuan and Zhou, Sihang and Zhao, Long and Liu, Zhe and Liu, Xinwang},
  journal={IEEE Transactions on Neural Networks and Learning Systems},
  volume={35},
  number={6},
  pages={7792--7803},
  year={2022},
  publisher={IEEE}
}

@inproceedings{yang2023cluster,
  title={Cluster-guided contrastive graph clustering network},
  author={Yang, Xihong and Liu, Yue and Zhou, Sihang and Wang, Siwei and Tu, Wenxuan and Zheng, Qun and Liu, Xinwang and Fang, Liming and Zhu, En},
  booktitle={Proceedings of the AAAI conference on artificial intelligence},
  volume={37},
  number={9},
  pages={10834--10842},
  year={2023}
}

@article{lin2021multi,
  title={Multi-view attributed graph clustering},
  author={Lin, Zhiping and Kang, Zhao and Zhang, Lizong and Tian, Ling},
  journal={IEEE Transactions on knowledge and data engineering},
  volume={35},
  number={2},
  pages={1872--1880},
  year={2021},
  publisher={IEEE}
}

@inproceedings{tu2024attribute,
  title={Attribute-missing graph clustering network},
  author={Tu, Wenxuan and Guan, Renxiang and Zhou, Sihang and Ma, Chuan and Peng, Xin and Cai, Zhiping and Liu, Zhe and Cheng, Jieren and Liu, Xinwang},
  booktitle={Proceedings of the AAAI Conference on Artificial Intelligence},
  volume={38},
  number={14},
  pages={15392--15401},
  year={2024}
}

@misc{fgl_unlearning,
      title={Federated Graph Unlearning}, 
      author={Yuming Ai and Xunkai Li and Jiaqi Chao and Bowen Fan and Zhengyu Wu and Yinlin Zhu and Rong-Hua Li and Guoren Wang},
      year={2025},
      eprint={2508.02485},
      archivePrefix={arXiv},
      primaryClass={cs.LG},
      url={https://arxiv.org/abs/2508.02485}, 
}

@article{graph_clustering_survey1,
  title={A survey of deep graph clustering: Taxonomy, challenge, application, and open resource},
  author={Liu, Yue and Xia, Jun and Zhou, Sihang and Yang, Xihong and Liang, Ke and Fan, Chenchen and Zhuang, Yan and Li, Stan Z and Liu, Xinwang and He, Kunlun},
  journal={arXiv preprint arXiv:2211.12875},
  year={2022}
}

@article{graph_clustering_survey2,
  title={Clustering attributed graphs: models, measures and methods},
  author={Bothorel, C{\'e}cile and Cruz, Juan David and Magnani, Matteo and Micenkova, Barbora},
  journal={Network Science},
  volume={3},
  number={3},
  pages={408--444},
  year={2015},
  publisher={Cambridge University Press}
}

@article{clip,
  author       = {Alec Radford and
                  Jong Wook Kim and
                  Chris Hallacy and
                  Aditya Ramesh and
                  Gabriel Goh and
                  Sandhini Agarwal and
                  Girish Sastry and
                  Amanda Askell and
                  Pamela Mishkin and
                  Jack Clark and
                  Gretchen Krueger and
                  Ilya Sutskever},
  title        = {Learning Transferable Visual Models From Natural Language Supervision},
  journal      = {CoRR},
  volume       = {abs/2103.00020},
  year         = {2021},
  url          = {https://arxiv.org/abs/2103.00020},
  eprinttype    = {arXiv},
  eprint       = {2103.00020},
  bibsource    = {dblp computer science bibliography, https://dblp.org}
}

@article{bt_cv_contrastive,
  author       = {Jure Zbontar and
                  Li Jing and
                  Ishan Misra and
                  Yann LeCun and
                  St{\'{e}}phane Deny},
  title        = {Barlow Twins: Self-Supervised Learning via Redundancy Reduction},
  journal      = {CoRR},
  volume       = {abs/2103.03230},
  year         = {2021},
  url          = {https://arxiv.org/abs/2103.03230},
  eprinttype    = {arXiv},
  eprint       = {2103.03230},
  bibsource    = {dblp computer science bibliography, https://dblp.org}
}

@article{graph_con_1,
  title={Relational symmetry based knowledge graph contrastive learning},
  author={Liang, Ke and Liu, Yue and Zhou, Sihang and Liu, Xinwang and Tu, Wenxuan},
  journal={arXiv preprint arXiv:2211.10738},
  year={2022}
}

@article{graph_con_2,
  title={An empirical study of graph contrastive learning},
  author={Zhu, Yanqiao and Xu, Yichen and Liu, Qiang and Wu, Shu},
  journal={arXiv preprint arXiv:2109.01116},
  year={2021}
}

@inproceedings{mcqueen1967some,
  title={Some methods of classification and analysis of multivariate observations},
  author={McQueen, James B},
  booktitle={Proc. of 5th Berkeley Symposium on Math. Stat. and Prob.},
  pages={281--297},
  year={1967}
}

@article{gcn,
  title={Semi-supervised classification with graph convolutional networks},
  author={Kipf, Thomas N and Welling, Max},
  journal={arXiv preprint arXiv:1609.02907},
  year={2016}
}

@inproceedings{sgc,
  title={Simplifying graph convolutional networks},
  author={Wu, Felix and Souza, Amauri and Zhang, Tianyi and Fifty, Christopher and Yu, Tao and Weinberger, Kilian},
  booktitle={International conference on machine learning},
  pages={6861--6871},
  year={2019},
  organization={Pmlr}
}

@inproceedings{optuna,
  title={Optuna: A next-generation hyperparameter optimization framework},
  author={Akiba, Takuya and Sano, Shotaro and Yanase, Toshihiko and Ohta, Takeru and Koyama, Masanori},
  booktitle={Proceedings of the 25th ACM SIGKDD international conference on knowledge discovery \& data mining, KDD},
  pages={2623--2631},
  year={2019}
}

@book{matrix_analysis,
  title={Matrix analysis},
  author={Horn, Roger A and Johnson, Charles R},
  year={2012},
  publisher={Cambridge university press}
}

@inproceedings{scaffold,
  title={Scaffold: Stochastic controlled averaging for federated learning},
  author={Karimireddy, Sai Praneeth and Kale, Satyen and Mohri, Mehryar and Reddi, Sashank and Stich, Sebastian and Suresh, Ananda Theertha},
  booktitle={International conference on machine learning},
  pages={5132--5143},
  year={2020},
  organization={PMLR}
}

@article{fedopt,
  title={Fedopt: Towards communication efficiency and privacy preservation in federated learning},
  author={Asad, Muhammad and Moustafa, Ahmed and Ito, Takayuki},
  journal={Applied Sciences},
  volume={10},
  number={8},
  pages={2864},
  year={2020},
  publisher={MDPI}
}

@article{fedper,
  title={Federated learning with personalization layers},
  author={Arivazhagan, Manoj Ghuhan and Aggarwal, Vinay and Singh, Aaditya Kumar and Choudhary, Sunav},
  journal={arXiv preprint arXiv:1912.00818},
  year={2019}
}

@article{fl_survey1,
  title={A survey on federated learning},
  author={Zhang, Chen and Xie, Yu and Bai, Hang and Yu, Bin and Li, Weihong and Gao, Yuan},
  journal={Knowledge-Based Systems},
  volume={216},
  pages={106775},
  year={2021},
  publisher={Elsevier}
}

@article{fl_survey2,
  title={A survey on federated learning: challenges and applications},
  author={Wen, Jie and Zhang, Zhixia and Lan, Yang and Cui, Zhihua and Cai, Jianghui and Zhang, Wensheng},
  journal={International journal of machine learning and cybernetics},
  volume={14},
  number={2},
  pages={513--535},
  year={2023},
  publisher={Springer}
}

@inproceedings{dennis2021heterogeneity,
  title={Heterogeneity for the win: One-shot federated clustering},
  author={Dennis, Don Kurian and Li, Tian and Smith, Virginia},
  booktitle={International conference on machine learning},
  pages={2611--2620},
  year={2021},
  organization={PMLR}
}

@inproceedings{moon,
  title={Model-contrastive federated learning},
  author={Li, Qinbin and He, Bingsheng and Song, Dawn},
  booktitle={Proceedings of the IEEE/CVF conference on computer vision and pattern recognition},
  pages={10713--10722},
  year={2021}
}

@inproceedings{feddc,
  title={Feddc: Federated learning with non-iid data via local drift decoupling and correction},
  author={Gao, Liang and Fu, Huazhu and Li, Li and Chen, Yingwen and Xu, Ming and Xu, Cheng-Zhong},
  booktitle={Proceedings of the IEEE/CVF conference on computer vision and pattern recognition},
  pages={10112--10121},
  year={2022}
}

@article{hu2020ogb,
  title={Open graph benchmark: Datasets for machine learning on graphs},
  author={Hu, Weihua and Fey, Matthias and Zitnik, Marinka and Dong, Yuxiao and Ren, Hongyu and Liu, Bowen and Catasta, Michele and Leskovec, Jure},
  journal={Advances in Neural Information Processing Systems, NeurIPS},
  year={2020}
}

@article{graphsage,
  title={Inductive representation learning on large graphs},
  author={Hamilton, Will and Ying, Zhitao and Leskovec, Jure},
  journal={Advances in neural information processing systems},
  volume={30},
  year={2017}
}

\appendix
\newpage
\section{More Related Works}

\textbf{Federated Learning.}
Federated learning (FL) is a distributed machine learning paradigm that enables multiple clients to collaboratively train a shared model without exchanging raw data, thereby preserving data privacy~\cite{fedavg, fedprox}. The canonical FedAvg algorithm~\cite{fedavg} aggregates locally trained model parameters via weighted averaging. Subsequent works address key challenges including statistical heterogeneity~\cite{fedprox, scaffold}, communication efficiency~\cite{fedopt}, and personalization~\cite{fedper}. Comprehensive reviews can be found in recent surveys~\cite{fl_survey1, fl_survey2}.

\vspace{+0.1cm}
\noindent \textbf{Federated Clustering in Non-Graph Domains.}
Federated clustering extends unsupervised learning to decentralized settings, aiming to discover latent patterns across distributed datasets. Early works adapt classical clustering algorithms to federated scenarios. For instance, $k$-FED~\cite{dennis2021heterogeneity} develops a heterogeneity-aware federated k-means that accounts for non-IID data distributions. However, the assumption of fixed cluster cardinality across clients is violated in federated graph scenarios, where subgraphs exhibit diverse community distributions.

\section{Dataset Details}
\label{appendix: dataset details}

This section provides detailed information about the datasets used in all experiments.

\vspace{+0.1cm}
\noindent \textbf{CiteSeer and PubMed}~\cite{Yang16cora} are widely used citation network datasets, where nodes represent papers and edges denote citation relationships. Node features are bag-of-words vectors indicating the presence or absence of specific words in each paper. These datasets are frequently used for node classification tasks.

\vspace{+0.1cm}
\noindent \textbf{Amazon Photo and Amazon Computers}~\cite{shchur2018amazon_datasets} are subsets of the Amazon co-purchase graph, where nodes represent individual products, and edges signify that two products are frequently bought together. Node features are derived from product reviews, represented as bag-of-words vectors capturing the textual information associated with each item. These datasets are commonly used for graph-based tasks such as node classification in recommendation systems.

\vspace{+0.1cm}
\noindent \textbf{Questions}~\cite{platonov2023hete_gnn_survey4} is derived from the question-answering platform Yandex Q. Nodes represent users, and an edge exists between two nodes if one user answers another user's question within a one-year timeframe (from September 2021 to August 2022). The objective is to predict which users remained active on the website (i.e., were not deleted or blocked) by the end of the specified period. Node features utilize the average FastText embeddings of words found in user descriptions.

\vspace{+0.1cm}
\noindent \textbf{ogb-arxiv}~\cite{hu2020ogb} is a citation network from the Microsoft Academic Graph where nodes represent arXiv computer science papers and edges denote citation relationships. The dataset contains 169,343 nodes and 1,166,243 edges, with 128-dimensional node features derived from paper titles and abstracts. Papers are grouped into 40 subject areas, making it suitable for evaluating graph clustering methods on large-scale academic networks.

\vspace{+0.1cm}
\noindent \textbf{ogb-products}~\cite{hu2020ogb} is an Amazon product co-purchasing network in which nodes represent products and edges indicate that two products are frequently bought together. It contains 2,449,029 nodes and 61,859,140 edges, with 100-dimensional bag-of-words features from product descriptions. Products belong to 47 categories, providing a large and diverse benchmark for graph learning.

\vspace{+0.1cm}
\noindent \textbf{Reddit}~\cite{graphsage} is a social network dataset where nodes represent Reddit posts and edges connect posts if a user comments on another post. The graph includes 232,965 nodes and 11,606,919 edges, with 602-dimensional content-based features. Nodes are associated with 41 communities (subreddits), making the dataset useful for evaluating graph clustering in social network settings.

\section{More Experimental Setups}
\label{appendix: more experimental setups}

\textbf{Hyperparameters.} For our proposed AdaFGC, the number of over-complete global community anchors is set to 10 times the real number of classes for all datasets. The smoothness rate $\alpha$ is searched within $\{10^{-2}, 10^{-1}, 1, 10, 100\}$, the contrastive learning temperature $\tau$ is fixed to $1$, and the propagation depth $L$ is fixed to $4$. Trade-off parameters $\lambda_\mathrm{anc}, \lambda_\mathrm{node}, \lambda_\mathrm{topo}$ are searched within $\{5\times10^{-2}, 1\times 10^{-1}, 5\times 10^{-1}, 1, 5, 10\}$, and the merge threshold $\eta$ is searched from $\{1\times 10^{-8}, 1\times 10^{-6}, 1\times 10^{-4}\}$. For the topology-level contrastive loss $\mathcal{L}_{\mathrm{topo}}$, we employ random walks with the following parameters: walks per node is fixed to $2$, walk length is fixed to $2$, and context size is fixed to $3$. For baselines, we adopt the hyperparameter configurations reported in their original papers whenever available. When unspecified, we employ automated hyperparameter optimization using the Optuna framework~\cite{optuna}. Unless the corresponding method specifies a particular model architecture, both the client and server employ a four-layer GNN to generate node embeddings, with hidden layer dimensions of 500-500-2000-10. Additionally, a one-layer MLP is used to produce local clustering signals, which are subsequently uploaded to the server. We conduct 50 communication rounds, each comprising 10 local training epochs. For model optimization, we adopt the Adam optimizer with a learning rate of $1\times 10^{-3}$. Results are reported as the mean and standard deviation across 10 runs.

\section{More Experiments}
\label{appendix: more experiments}

\subsection{Varying Participating Clients}

To investigate how AdaFGC performs under different client participation settings, we further vary the number of participating clients from $K \in \{5, 20\}$ by partitioning the datasets accordingly. Each client receives a subgraph with approximately equal number of nodes. We evaluate clustering performance using ACC, NMI, ARI, and F1 metrics across all five datasets.

The experimental results are presented in Table.~\ref{tab: vary_5} and Table.~\ref{tab: vary_20}. As observed, AdaFGC consistently achieves the best performance across all datasets and metrics under both client settings. When $K=5$, AdaFGC outperforms the second-best baseline FedNCN by substantial margins, with ACC improvements of +19.87\% on CiteSeer and +17.79 on Amazon-Computer. For NMI, which measures the mutual information between predicted and ground-truth clusters, AdaFGC demonstrates even more significant gains of +15.87\% on CiteSeer and +5.77\% on Amazon-Computer. The ARI results show improvements of +2.96\% on CiteSeer and +3.53\% on Amazon-Computer. When $K=20$, AdaFGC achieves ACC gains of +4.83\% on CiteSeer and +8.40\% on Amazon-Computer, with corresponding NMI improvements of +2.23\% and +2.91\%, respectively. Notably, traditional federated graph learning methods exhibit significantly inferior performance, with many achieving near-random clustering results.
\begin{table*}[htbp]
    \setlength{\abovecaptionskip}{0.2cm}
    \setlength{\belowcaptionskip}{-0.2cm}
    \centering
    \caption{\textbf{Node clustering performance} on five datasets (5 clients). The best, second best and third best results are highlighted in \textcolor{darkred}{\textbf{red}}, \textcolor{royalblue}{\textbf{blue}} and \textcolor{orange}{\textbf{orange}}, respectively.}
    \label{tab: vary_5}
    \footnotesize 
    \renewcommand{\arraystretch}{1.1}
    \resizebox{\linewidth}{!}{
    \setlength{\tabcolsep}{1.2mm}{
    \begin{tabular}{c!{\vrule width 0.1pt}
    cc!{\vrule width 0.1pt}
    cc!{\vrule width 0.1pt}
    cc!{\vrule width 0.1pt}
    cc!{\vrule width 0.1pt}
    cc}
    \hline\thickhline
    \rowcolor{gray!10}
    
    & \multicolumn{2}{c!{\vrule width 0.1pt}}{\textbf{CiteSeer}}  
    & \multicolumn{2}{c!{\vrule width 0.1pt}}{\textbf{PubMed}}  
    & \multicolumn{2}{c!{\vrule width 0.1pt}}{\textbf{Amazon-Computer}}  
    & \multicolumn{2}{c!{\vrule width 0.1pt}}{\textbf{Amazon-Photo}}  
    & \multicolumn{2}{c}{\textbf{Questions}}  \\
    
    \cline{2-11}
    \rowcolor{gray!10}
    \multirow{-2}{*}{\diagbox[width=8.5em,height=2.4em]{\textbf{Methods}}{\textbf{Datasets}}}  
    & ACC & NMI 
    & ACC & NMI 
    & ACC & NMI 
    & ACC & NMI 
    & ACC & NMI \\
    
    \hline
    FedSage+$^*$ & $15.82_{\pm 1.68}$ & $2.58_{\pm 0.38}$ & $48.91_{\pm 2.21}$ & $8.42_{\pm 1.65}$ & $21.18_{\pm 0.89}$ & $4.96_{\pm 1.69}$ & $31.87_{\pm 1.84}$ & $10.68_{\pm 2.18}$ & $77.52_{\pm 2.03}$ & \textcolor{orange}{$\mathbf{0.91_{\pm 0.46}}$} \\
    \rowcolor{gray!10}
    FedPUB$^*$ & $16.52_{\pm 3.31}$ & $0.02_{\pm 0.00}$ & \textcolor{orange}{$\mathbf{51.28_{\pm 9.28}}$} & $0.00_{\pm 0.01}$ & $11.48_{\pm 12.68}$ & $0.00_{\pm 0.00}$ & $14.56_{\pm 7.46}$ & $0.00_{\pm 0.00}$ & $58.73_{\pm 15.12}$ & $0.00_{\pm 0.00}$ \\
    FedTAD$^*$ & $17.24_{\pm 3.25}$ & $1.18_{\pm 1.39}$ & $35.73_{\pm 2.43}$ & $1.13_{\pm 0.96}$ & $9.61_{\pm 4.28}$ & $0.00_{\pm 0.00}$ & $13.21_{\pm 4.87}$ & $4.02_{\pm 6.64}$ & $55.38_{\pm 17.18}$ & $0.17_{\pm 0.13}$ \\
    \rowcolor{gray!10}
    FedGTA$^*$ & \textcolor{orange}{$\mathbf{23.58_{\pm 0.78}}$} & $5.54_{\pm 1.03}$ & $45.67_{\pm 2.44}$ & $6.01_{\pm 1.88}$ & $20.98_{\pm 1.72}$ & \textcolor{orange}{$\mathbf{11.79_{\pm 2.86}}$} & \textcolor{orange}{$\mathbf{34.82_{\pm 1.91}}$} & \textcolor{orange}{$\mathbf{13.97_{\pm 2.95}}$} & $87.16_{\pm 4.35}$ & $0.00_{\pm 0.09}$ \\
    FedIIH$^*$ & $13.92_{\pm 2.84}$ & $0.00_{\pm 0.00}$ & $17.23_{\pm 5.94}$ & $0.04_{\pm 0.09}$ & $13.24_{\pm 2.11}$ & $0.00_{\pm 0.00}$ & $14.05_{\pm 2.11}$ & $0.00_{\pm 0.00}$ & $75.83_{\pm 3.09}$ & $0.00_{\pm 0.00}$ \\
    \rowcolor{gray!10}
    FGSSL$^*$ & $11.35_{\pm 2.02}$ & $2.41_{\pm 1.60}$ & $40.85_{\pm 3.58}$ & $7.59_{\pm 2.99}$ & $16.69_{\pm 2.56}$ & $2.89_{\pm 2.22}$ & $22.41_{\pm 2.90}$ & $3.50_{\pm 1.50}$ & \textcolor{orange}{$\mathbf{93.27_{\pm 4.00}}$} & $0.24_{\pm 0.22}$ \\
    FGGP$^*$ & $11.55_{\pm 2.14}$ & $2.28_{\pm 0.49}$ & $49.42_{\pm 2.63}$ & $7.63_{\pm 2.62}$ & \textcolor{orange}{$\mathbf{21.93_{\pm 4.01}}$} & $5.05_{\pm 0.92}$ & $17.94_{\pm 1.26}$ & $4.17_{\pm 1.05}$ & $84.89_{\pm 3.80}$ & $0.34_{\pm 0.50}$ \\
    \rowcolor{gray!10}
    FedSPA$^*$ & $21.59_{\pm 2.67}$ & \textcolor{orange}{$\mathbf{6.06_{\pm 2.98}}$} & $43.47_{\pm 3.95}$ & \textcolor{orange}{$\mathbf{8.97_{\pm 1.05}}$} & $12.34_{\pm 3.35}$ & $2.75_{\pm 2.62}$ & $20.47_{\pm 1.38}$ & $3.66_{\pm 0.08}$ & $83.65_{\pm 6.61}$ & $0.00_{\pm 0.07}$ \\
    S2FGL$^*$ & $12.43_{\pm 2.42}$ & $2.33_{\pm 2.71}$ & $38.61_{\pm 2.00}$ & $7.24_{\pm 1.67}$ & $10.62_{\pm 3.87}$ & $2.05_{\pm 2.67}$ & $11.32_{\pm 3.65}$ & $1.78_{\pm 0.64}$ & $69.78_{\pm 7.44}$ & $0.00_{\pm 0.09}$ \\
    \hline
    \rowcolor{gray!10}
    FedNCN & \textcolor{royalblue}{$\mathbf{56.21_{\pm 1.82}}$} & \textcolor{royalblue}{$\mathbf{13.58_{\pm 2.51}}$} & \textcolor{royalblue}{$\mathbf{64.87_{\pm 1.42}}$} & \textcolor{royalblue}{$\mathbf{10.81_{\pm 1.52}}$} & \textcolor{royalblue}{$\mathbf{65.63_{\pm 1.22}}$} & \textcolor{royalblue}{$\mathbf{24.01_{\pm 1.03}}$} & \textcolor{royalblue}{$\mathbf{75.72_{\pm 3.02}}$} & \textcolor{royalblue}{$\mathbf{38.09_{\pm 5.47}}$} & \textcolor{royalblue}{$\mathbf{96.81_{\pm 2.09}}$} & \textcolor{royalblue}{$\mathbf{1.93_{\pm 0.88}}$} \\
    AdaFGC (Ours) & \textcolor{darkred}{$\mathbf{76.08_{\pm 1.45}}$} & \textcolor{darkred}{$\mathbf{29.45_{\pm 1.77}}$} & \textcolor{darkred}{$\mathbf{77.13_{\pm 1.54}}$} & \textcolor{darkred}{$\mathbf{17.12_{\pm 2.18}}$} & \textcolor{darkred}{$\mathbf{83.42_{\pm 2.33}}$} & \textcolor{darkred}{$\mathbf{29.78_{\pm 2.88}}$} & \textcolor{darkred}{$\mathbf{84.11_{\pm 1.14}}$} & \textcolor{darkred}{$\mathbf{44.05_{\pm 7.19}}$} & \textcolor{darkred}{$\mathbf{97.54_{\pm 3.00}}$} & \textcolor{darkred}{$\mathbf{6.59_{\pm 3.00}}$} \\
    \hline\thickhline
    \rowcolor{gray!10}
    
    & \multicolumn{2}{c!{\vrule width 0.1pt}}{\textbf{CiteSeer}}  
    & \multicolumn{2}{c!{\vrule width 0.1pt}}{\textbf{PubMed}}  
    & \multicolumn{2}{c!{\vrule width 0.1pt}}{\textbf{Amazon-Computer}}  
    & \multicolumn{2}{c!{\vrule width 0.1pt}}{\textbf{Amazon-Photo}}  
    & \multicolumn{2}{c}{\textbf{Questions}}  \\
    
    \cline{2-11}
    \rowcolor{gray!10}
    \multirow{-2}{*}{\diagbox[width=8.5em,height=2.4em]{\textbf{Methods}}{\textbf{Datasets}}}  
    & ARI & F1 
    & ARI & F1 
    & ARI & F1 
    & ARI & F1 
    & ARI & F1 \\
    
    \hline
    FedSage+$^*$ & $0.17_{\pm 0.57}$ & $11.18_{\pm 0.73}$ & $0.03_{\pm 0.01}$ & $25.63_{\pm 2.88}$ & $3.89_{\pm 2.05}$ & $14.53_{\pm 2.01}$ & $11.19_{\pm 2.67}$ & \textcolor{orange}{$\mathbf{18.33_{\pm 1.09}}$} & $1.28_{\pm 0.48}$ & $38.42_{\pm 0.78}$ \\
    \rowcolor{gray!10}
    FedPUB$^*$ & $0.00_{\pm 0.00}$ & $4.02_{\pm 0.55}$ & $0.00_{\pm 0.00}$ & $13.71_{\pm 3.12}$ & $0.00_{\pm 0.00}$ & $2.24_{\pm 2.35}$ & $0.00_{\pm 0.00}$ & $3.37_{\pm 1.70}$ & $0.00_{\pm 0.00}$ & $30.07_{\pm 12.21}$ \\
    FedTAD$^*$ & $0.59_{\pm 0.98}$ & $6.98_{\pm 2.49}$ & $1.54_{\pm 1.24}$ & $9.68_{\pm 2.63}$ & $0.00_{\pm 0.00}$ & $2.56_{\pm 1.51}$ & $0.00_{\pm 0.09}$ & $1.95_{\pm 1.13}$ & $-0.42_{\pm 0.33}$ & $35.07_{\pm 12.07}$ \\
    \rowcolor{gray!10}
    FedGTA$^*$ & $3.46_{\pm 0.68}$ & \textcolor{orange}{$\mathbf{19.53_{\pm 1.03}}$} & $4.64_{\pm 1.91}$ & $19.26_{\pm 0.70}$ & \textcolor{orange}{$\mathbf{8.20_{\pm 2.20}}$} & $8.77_{\pm 0.99}$ & \textcolor{orange}{$\mathbf{11.63_{\pm 2.00}}$} & $11.73_{\pm 1.13}$ & \textcolor{orange}{$\mathbf{1.73_{\pm 0.60}}$} & $49.44_{\pm 1.73}$ \\
    FedIIH$^*$ & $0.00_{\pm 0.26}$ & $3.65_{\pm 0.61}$ & $0.02_{\pm 0.03}$ & $14.33_{\pm 2.78}$ & $0.00_{\pm 0.00}$ & $1.83_{\pm 0.92}$ & $0.00_{\pm 0.00}$ & $3.02_{\pm 1.51}$ & $0.00_{\pm 0.00}$ & $30.04_{\pm 12.26}$ \\
    \rowcolor{gray!10}
    FGSSL$^*$ & $1.92_{\pm 0.91}$ & $8.58_{\pm 2.53}$ & $8.41_{\pm 0.38}$ & $30.01_{\pm 2.88}$ & $3.28_{\pm 0.01}$ & $11.17_{\pm 1.58}$ & $3.86_{\pm 0.04}$ & $15.75_{\pm 0.56}$ & $0.15_{\pm 0.30}$ & $37.18_{\pm 2.76}$ \\
    FGGP$^*$ & $2.53_{\pm 1.80}$ & $8.42_{\pm 1.20}$ & $5.63_{\pm 0.84}$ & \textcolor{orange}{$\mathbf{36.49_{\pm 3.27}}$} & $2.63_{\pm 2.03}$ & \textcolor{orange}{$\mathbf{16.87_{\pm 3.70}}$} & $3.70_{\pm 0.54}$ & $12.67_{\pm 0.55}$ & $1.38_{\pm 0.11}$ & $46.88_{\pm 3.98}$ \\
    \rowcolor{gray!10}
    FedSPA$^*$ & \textcolor{orange}{$\mathbf{4.41_{\pm 1.20}}$} & $15.14_{\pm 1.41}$ & \textcolor{orange}{$\mathbf{10.04_{\pm 0.46}}$} & $30.48_{\pm 2.45}$ & $1.92_{\pm 1.34}$ & $8.79_{\pm 0.72}$ & $3.90_{\pm 2.18}$ & $12.90_{\pm 2.14}$ & $0.28_{\pm 0.67}$ & $37.09_{\pm 4.70}$ \\
    S2FGL$^*$ & $2.56_{\pm 0.91}$ & $9.07_{\pm 1.55}$ & $7.13_{\pm 0.36}$ & $24.04_{\pm 5.07}$ & $2.00_{\pm 2.00}$ & $7.73_{\pm 1.97}$ & $1.86_{\pm 0.49}$ & $7.06_{\pm 1.94}$ & $0.40_{\pm 0.52}$ & \textcolor{orange}{$\mathbf{49.94_{\pm 4.12}}$} \\
    \hline
    \rowcolor{gray!10}
    FedNCN & \textcolor{royalblue}{$\mathbf{20.12_{\pm 3.80}}$} & \textcolor{royalblue}{$\mathbf{25.36_{\pm 1.45}}$} & \textcolor{royalblue}{$\mathbf{16.14_{\pm 4.11}}$} & \textcolor{royalblue}{$\mathbf{27.78_{\pm 2.99}}$} & \textcolor{royalblue}{$\mathbf{22.13_{\pm 2.33}}$} & \textcolor{royalblue}{$\mathbf{25.09_{\pm 2.18}}$} & \textcolor{royalblue}{$\mathbf{44.08_{\pm 7.18}}$} & \textcolor{royalblue}{$\mathbf{29.78_{\pm 2.88}}$} & \textcolor{royalblue}{$\mathbf{6.59_{\pm 3.00}}$} & \textcolor{royalblue}{$\mathbf{54.11_{\pm 1.77}}$} \\
    AdaFGC (Ours) & \textcolor{darkred}{$\mathbf{23.08_{\pm 3.51}}$} & \textcolor{darkred}{$\mathbf{32.89_{\pm 2.54}}$} & \textcolor{darkred}{$\mathbf{19.40_{\pm 4.23}}$} & \textcolor{darkred}{$\mathbf{43.63_{\pm 3.20}}$} & \textcolor{darkred}{$\mathbf{25.66_{\pm 3.99}}$} & \textcolor{darkred}{$\mathbf{34.72_{\pm 0.71}}$} & \textcolor{darkred}{$\mathbf{47.35_{\pm 6.76}}$} & \textcolor{darkred}{$\mathbf{33.95_{\pm 2.70}}$} & \textcolor{darkred}{$\mathbf{7.77_{\pm 2.37}}$} & \textcolor{darkred}{$\mathbf{61.08_{\pm 2.16}}$} \\
    \hline\thickhline
    \end{tabular}
    }}
    \end{table*}

\begin{table*}[htbp]
    \setlength{\abovecaptionskip}{0.2cm}
    \setlength{\belowcaptionskip}{-0.2cm}
    \centering
    \caption{\textbf{Node clustering performance} on five datasets (20 clients). The best, second best and third best results are highlighted in \textcolor{darkred}{\textbf{red}}, \textcolor{royalblue}{\textbf{blue}} and \textcolor{orange}{\textbf{orange}}, respectively.}
    \label{tab: vary_20}
    \footnotesize 
    \renewcommand{\arraystretch}{1.1}
    \resizebox{\linewidth}{!}{
    \setlength{\tabcolsep}{1.2mm}{
    \begin{tabular}{c!{\vrule width 0.1pt}
    cc!{\vrule width 0.1pt}
    cc!{\vrule width 0.1pt}
    cc!{\vrule width 0.1pt}
    cc!{\vrule width 0.1pt}
    cc}
    \hline\thickhline
    \rowcolor{gray!10}
    
    & \multicolumn{2}{c!{\vrule width 0.1pt}}{\textbf{CiteSeer}}  
    & \multicolumn{2}{c!{\vrule width 0.1pt}}{\textbf{PubMed}}  
    & \multicolumn{2}{c!{\vrule width 0.1pt}}{\textbf{Amazon-Computer}}  
    & \multicolumn{2}{c!{\vrule width 0.1pt}}{\textbf{Amazon-Photo}}  
    & \multicolumn{2}{c}{\textbf{Questions}}  \\
    
    \cline{2-11}
    \rowcolor{gray!10}
    \multirow{-2}{*}{\diagbox[width=8.5em,height=2.4em]{\textbf{Methods}}{\textbf{Datasets}}}  
    & ACC & NMI 
    & ACC & NMI 
    & ACC & NMI 
    & ACC & NMI 
    & ACC & NMI \\
    
    \hline
    FedSage+$^*$ & $15.28_{\pm 1.99}$ & $9.35_{\pm 0.20}$ & $42.96_{\pm 3.16}$ & $1.92_{\pm 1.70}$ & $20.26_{\pm 1.25}$ & $6.93_{\pm 1.33}$ & \textcolor{orange}{$\mathbf{36.04_{\pm 0.95}}$} & $11.05_{\pm 0.23}$ & $74.56_{\pm 0.62}$ & $0.63_{\pm 0.71}$ \\
    \rowcolor{gray!10}
    FedPUB$^*$ & $10.96_{\pm 6.35}$ & $0.00_{\pm 0.00}$ & $27.77_{\pm 9.07}$ & $0.00_{\pm 0.00}$ & $13.98_{\pm 11.90}$ & $0.00_{\pm 0.00}$ & $12.47_{\pm 6.67}$ & \textcolor{orange}{$\mathbf{14.89_{\pm 0.00}}$} & $39.73_{\pm 15.07}$ & $0.01_{\pm 0.02}$ \\
    FedTAD$^*$ & $17.43_{\pm 1.51}$ & $3.35_{\pm 3.52}$ & $35.82_{\pm 5.13}$ & $0.43_{\pm 0.68}$ & $14.51_{\pm 11.04}$ & $4.98_{\pm 1.31}$ & $11.55_{\pm 4.86}$ & $8.27_{\pm 1.02}$ & $67.75_{\pm 13.69}$ & $0.06_{\pm 0.07}$ \\
    \rowcolor{gray!10}
    FedGTA$^*$ & \textcolor{orange}{$\mathbf{27.52_{\pm 1.59}}$} & \textcolor{orange}{$\mathbf{12.10_{\pm 2.04}}$} & $47.83_{\pm 1.67}$ & $0.02_{\pm 0.72}$ & \textcolor{orange}{$\mathbf{21.66_{\pm 1.23}}$} & \textcolor{orange}{$\mathbf{7.10_{\pm 0.94}}$} & $28.83_{\pm 1.15}$ & $10.60_{\pm 0.74}$ & \textcolor{orange}{$\mathbf{85.03_{\pm 2.34}}$} & $0.71_{\pm 0.68}$ \\
    FedIIH$^*$ & $17.99_{\pm 7.37}$ & $0.00_{\pm 0.00}$ & \textcolor{orange}{$\mathbf{65.05_{\pm 1.69}}$} & $7.31_{\pm 0.85}$ & $7.43_{\pm 4.49}$ & $0.02_{\pm 0.03}$ & $5.15_{\pm 0.15}$ & $0.00_{\pm 0.01}$ & $48.84_{\pm 3.04}$ & \textcolor{orange}{$\mathbf{1.06_{\pm 0.79}}$} \\
    \rowcolor{gray!10}
    FGSSL$^*$ & $9.32_{\pm 2.74}$ & $1.97_{\pm 1.31}$ & $33.42_{\pm 3.66}$ & $6.22_{\pm 2.27}$ & $13.68_{\pm 2.10}$ & $2.37_{\pm 1.82}$ & $18.37_{\pm 2.47}$ & $2.87_{\pm 1.23}$ & $76.39_{\pm 3.76}$ & $0.20_{\pm 0.18}$ \\
    FGGP$^*$ & $9.46_{\pm 3.39}$ & $1.86_{\pm 0.40}$ & $40.47_{\pm 2.15}$ & $6.25_{\pm 2.78}$ & $17.96_{\pm 4.47}$ & $4.13_{\pm 0.75}$ & $14.69_{\pm 1.03}$ & $3.42_{\pm 0.86}$ & $69.51_{\pm 3.85}$ & $0.28_{\pm 0.23}$ \\
    \rowcolor{gray!10}
    FedSPA$^*$ & $17.68_{\pm 2.64}$ & $4.97_{\pm 2.44}$ & $35.59_{\pm 3.97}$ & \textcolor{orange}{$\mathbf{7.35_{\pm 0.86}}$} & $10.11_{\pm 3.20}$ & $2.25_{\pm 2.14}$ & $16.76_{\pm 1.13}$ & $2.99_{\pm 0.06}$ & $68.51_{\pm 8.19}$ & $0.00_{\pm 0.05}$ \\
    S2FGL$^*$ & $10.18_{\pm 3.25}$ & $1.91_{\pm 2.22}$ & $31.63_{\pm 1.64}$ & $5.93_{\pm 1.37}$ & $8.70_{\pm 3.90}$ & $1.68_{\pm 2.19}$ & $9.27_{\pm 3.08}$ & $1.46_{\pm 0.52}$ & $57.15_{\pm 8.86}$ & $0.00_{\pm 0.07}$ \\
    \hline
    \rowcolor{gray!10}
    FedNCN & \textcolor{royalblue}{$\mathbf{58.84_{\pm 1.59}}$} & \textcolor{royalblue}{$\mathbf{20.15_{\pm 1.97}}$} & \textcolor{royalblue}{$\mathbf{67.92_{\pm 1.76}}$} & \textcolor{royalblue}{$\mathbf{7.64_{\pm 0.89}}$} & \textcolor{royalblue}{$\mathbf{71.85_{\pm 2.67}}$} & \textcolor{royalblue}{$\mathbf{22.14_{\pm 1.52}}$} & \textcolor{royalblue}{$\mathbf{67.63_{\pm 2.11}}$} & \textcolor{royalblue}{$\mathbf{20.32_{\pm 1.66}}$} & \textcolor{royalblue}{$\mathbf{93.19_{\pm 5.26}}$} & \textcolor{royalblue}{$\mathbf{1.11_{\pm 0.83}}$} \\
    AdaFGC (Ours) & \textcolor{darkred}{$\mathbf{63.67_{\pm 2.29}}$} & \textcolor{darkred}{$\mathbf{22.38_{\pm 1.14}}$} & \textcolor{darkred}{$\mathbf{70.54_{\pm 1.31}}$} & \textcolor{darkred}{$\mathbf{13.21_{\pm 1.61}}$} & \textcolor{darkred}{$\mathbf{80.25_{\pm 1.61}}$} & \textcolor{darkred}{$\mathbf{25.05_{\pm 2.70}}$} & \textcolor{darkred}{$\mathbf{77.60_{\pm 1.89}}$} & \textcolor{darkred}{$\mathbf{30.68_{\pm 2.14}}$} & \textcolor{darkred}{$\mathbf{94.96_{\pm 4.93}}$} & \textcolor{darkred}{$\mathbf{3.92_{\pm 0.82}}$} \\
    \hline\thickhline
    \rowcolor{gray!10}
    
    & \multicolumn{2}{c!{\vrule width 0.1pt}}{\textbf{CiteSeer}}  
    & \multicolumn{2}{c!{\vrule width 0.1pt}}{\textbf{PubMed}}  
    & \multicolumn{2}{c!{\vrule width 0.1pt}}{\textbf{Amazon-Computer}}  
    & \multicolumn{2}{c!{\vrule width 0.1pt}}{\textbf{Amazon-Photo}}  
    & \multicolumn{2}{c}{\textbf{Questions}}  \\
    
    \cline{2-11}
    \rowcolor{gray!10}
    \multirow{-2}{*}{\diagbox[width=8.5em,height=2.4em]{\textbf{Methods}}{\textbf{Datasets}}}  
    & ARI & F1 
    & ARI & F1 
    & ARI & F1 
    & ARI & F1 
    & ARI & F1 \\
    
    \hline
    FedSage+$^*$ & $0.25_{\pm 0.51}$ & $9.86_{\pm 0.82}$ & $-0.02_{\pm 0.50}$ & \textcolor{orange}{$\mathbf{30.20_{\pm 3.14}}$} & \textcolor{orange}{$\mathbf{4.64_{\pm 1.96}}$} & \textcolor{orange}{$\mathbf{20.74_{\pm 0.91}}$} & \textcolor{orange}{$\mathbf{9.61_{\pm 0.33}}$} & \textcolor{orange}{$\mathbf{26.00_{\pm 0.53}}$} & $0.56_{\pm 0.94}$ & $35.53_{\pm 0.74}$ \\
    \rowcolor{gray!10}
    FedPUB$^*$ & $0.00_{\pm 0.00}$ & $3.68_{\pm 2.00}$ & $0.00_{\pm 0.00}$ & $13.72_{\pm 3.31}$ & $0.00_{\pm 0.00}$ & $3.94_{\pm 3.39}$ & $0.00_{\pm 0.00}$ & $5.90_{\pm 3.51}$ & $-0.04_{\pm 0.07}$ & $39.24_{\pm 0.75}$ \\
    FedTAD$^*$ & $1.06_{\pm 1.34}$ & $6.64_{\pm 2.35}$ & $0.53_{\pm 0.87}$ & $19.84_{\pm 3.79}$ & $3.64_{\pm 1.01}$ & $6.78_{\pm 0.51}$ & $0.58_{\pm 1.11}$ & $9.80_{\pm 0.00}$ & $-0.31_{\pm 0.40}$ & $37.98_{\pm 12.73}$ \\
    \rowcolor{gray!10}
    FedGTA$^*$ & \textcolor{orange}{$\mathbf{4.23_{\pm 2.04}}$} & \textcolor{orange}{$\mathbf{16.79_{\pm 2.17}}$} & $0.71_{\pm 1.17}$ & $17.82_{\pm 1.61}$ & $3.39_{\pm 0.75}$ & $6.17_{\pm 0.27}$ & $7.68_{\pm 0.85}$ & $9.41_{\pm 0.33}$ & \textcolor{orange}{$\mathbf{1.18_{\pm 0.37}}$} & \textcolor{orange}{$\mathbf{46.82_{\pm 0.41}}$} \\
    FedIIH$^*$ & $0.00_{\pm 0.00}$ & $12.18_{\pm 3.23}$ & $0.00_{\pm 0.00}$ & $14.00_{\pm 1.28}$ & $0.00_{\pm 0.00}$ & $0.05_{\pm 0.10}$ & $0.00_{\pm 0.00}$ & $2.33_{\pm 1.31}$ & $0.62_{\pm 1.30}$ & $40.80_{\pm 3.67}$ \\
    \rowcolor{gray!10}
    FGSSL$^*$ & $1.57_{\pm 0.74}$ & $7.03_{\pm 2.07}$ & $6.90_{\pm 0.31}$ & $24.59_{\pm 2.46}$ & $2.69_{\pm 0.01}$ & $9.15_{\pm 1.29}$ & $3.16_{\pm 0.04}$ & $12.90_{\pm 0.46}$ & $0.13_{\pm 0.24}$ & $30.46_{\pm 2.36}$ \\
    FGGP$^*$ & $2.07_{\pm 1.47}$ & $6.89_{\pm 0.99}$ & $4.61_{\pm 0.69}$ & $29.90_{\pm 3.14}$ & $2.15_{\pm 1.66}$ & $13.81_{\pm 3.03}$ & $3.03_{\pm 0.44}$ & $10.38_{\pm 0.45}$ & $1.13_{\pm 0.09}$ & $38.41_{\pm 3.99}$ \\
    \rowcolor{gray!10}
    FedSPA$^*$ & $3.61_{\pm 0.99}$ & $12.40_{\pm 1.16}$ & \textcolor{orange}{$\mathbf{8.22_{\pm 0.38}}$} & $24.98_{\pm 2.01}$ & $1.57_{\pm 1.09}$ & $7.20_{\pm 0.59}$ & $3.19_{\pm 2.60}$ & $10.57_{\pm 1.76}$ & $0.23_{\pm 0.55}$ & $30.38_{\pm 6.14}$ \\
    S2FGL$^*$ & $2.10_{\pm 0.74}$ & $7.43_{\pm 1.27}$ & $5.84_{\pm 0.30}$ & $19.69_{\pm 4.15}$ & $1.64_{\pm 1.64}$ & $6.33_{\pm 1.61}$ & $1.52_{\pm 0.40}$ & $5.78_{\pm 1.59}$ & $0.33_{\pm 0.43}$ & $40.91_{\pm 5.66}$ \\
    \hline
    \rowcolor{gray!10}
    FedNCN & \textcolor{royalblue}{$\mathbf{18.38_{\pm 1.76}}$} & \textcolor{royalblue}{$\mathbf{33.58_{\pm 1.12}}$} & \textcolor{royalblue}{$\mathbf{10.25_{\pm 1.24}}$} & \textcolor{royalblue}{$\mathbf{42.23_{\pm 1.34}}$} & \textcolor{royalblue}{$\mathbf{26.19_{\pm 4.46}}$} & \textcolor{royalblue}{$\mathbf{24.42_{\pm 3.28}}$} & \textcolor{royalblue}{$\mathbf{17.12_{\pm 2.06}}$} & \textcolor{royalblue}{$\mathbf{33.23_{\pm 0.94}}$} & \textcolor{royalblue}{$\mathbf{3.00_{\pm 3.36}}$} & \textcolor{royalblue}{$\mathbf{52.02_{\pm 2.34}}$} \\
    AdaFGC (Ours) & \textcolor{darkred}{$\mathbf{18.43_{\pm 3.46}}$} & \textcolor{darkred}{$\mathbf{36.00_{\pm 2.52}}$} & \textcolor{darkred}{$\mathbf{15.64_{\pm 4.20}}$} & \textcolor{darkred}{$\mathbf{47.08_{\pm 3.17}}$} & \textcolor{darkred}{$\mathbf{26.11_{\pm 3.95}}$} & \textcolor{darkred}{$\mathbf{29.35_{\pm 0.71}}$} & \textcolor{darkred}{$\mathbf{32.06_{\pm 3.45}}$} & \textcolor{darkred}{$\mathbf{38.75_{\pm 1.08}}$} & \textcolor{darkred}{$\mathbf{6.54_{\pm 2.35}}$} & \textcolor{darkred}{$\mathbf{53.44_{\pm 2.16}}$} \\
    \hline\thickhline
    \end{tabular}
    }}
    \end{table*}

\subsection{Comparison with on Non-Graph Baselines}

To demonstrate the necessity of graph-specific designs in federated clustering, we compare AdaFGC with non-graph federated learning baselines that treat node features as independent samples without modeling graph structures. The baselines include standard FL methods (FedAvg~\cite{fedavg}, FedProx~\cite{fedprox}, FedPer~\cite{fedper}, SCAFFOLD~\cite{scaffold}, MOON~\cite{moon}, and FedDC~\cite{feddc}), which aggregate model parameters without considering graph topology. We also include Local training (no federation) as a reference and adopt the same base encoder with 10 clients for all methods. For comparison, we further include FedNCN. As shown in Table~\ref{tab: non_graph}, non-graph FL methods achieve limited performance, indicating that treating nodes as independent samples fails to capture relational information in graphs. Moreover, Local training can even surpass some non-graph federated methods, implying that naive federation without graph awareness may be counterproductive. In contrast, graph-aware approaches (FedNCN and AdaFGC) consistently outperform all non-graph baselines, demonstrating that explicitly modeling graph structures and node relationships is essential for effective federated graph clustering.

\begin{table*}[htbp]
    \setlength{\abovecaptionskip}{0.2cm}
    \setlength{\belowcaptionskip}{-0.2cm}
    \centering
    \caption{\textbf{Performance comparison of different non-graph baselines} on five datasets (10 clients). The best, second best and third best results are highlighted in \textcolor{darkred}{\textbf{red}}, \textcolor{royalblue}{\textbf{blue}} and \textcolor{orange}{\textbf{orange}}, respectively.}
    \label{tab: non_graph}
    \footnotesize 
    \renewcommand{\arraystretch}{1.1}
    \resizebox{\linewidth}{!}{
    \setlength{\tabcolsep}{1.2mm}{
    \begin{tabular}{c!{\vrule width 0.1pt}
    cc!{\vrule width 0.1pt}
    cc!{\vrule width 0.1pt}
    cc!{\vrule width 0.1pt}
    cc!{\vrule width 0.1pt}
    cc}
    \hline\thickhline
    \rowcolor{gray!10}
    
    & \multicolumn{2}{c!{\vrule width 0.1pt}}{\textbf{CiteSeer}}  
    & \multicolumn{2}{c!{\vrule width 0.1pt}}{\textbf{PubMed}}  
    & \multicolumn{2}{c!{\vrule width 0.1pt}}{\textbf{Amazon-Computer}}  
    & \multicolumn{2}{c!{\vrule width 0.1pt}}{\textbf{Amazon-Photo}}  
    & \multicolumn{2}{c}{\textbf{Questions}}  \\
    
    \cline{2-11}
    \rowcolor{gray!10}
    \multirow{-2}{*}{\diagbox[width=8.5em,height=2.4em]{\textbf{Methods}}{\textbf{Datasets}}}  
    & ACC & NMI 
    & ACC & NMI 
    & ACC & NMI 
    & ACC & NMI 
    & ACC & NMI \\
    
    \hline
    Local & $43.68_{\pm 1.65}$ & \textcolor{orange}{$\mathbf{8.42_{\pm 1.17}}$} & $53.76_{\pm 1.46}$ & \textcolor{orange}{$\mathbf{7.84_{\pm 1.86}}$} & $45.32_{\pm 0.29}$ & \textcolor{orange}{$\mathbf{4.70_{\pm 0.10}}$} & $47.64_{\pm 0.78}$ & \textcolor{orange}{$\mathbf{5.58_{\pm 0.49}}$} & $77.12_{\pm 1.76}$ & \textcolor{orange}{$\mathbf{0.49_{\pm 0.10}}$} \\
    \rowcolor{gray!10}
    FedAvg & \textcolor{orange}{$\mathbf{55.23_{\pm 2.73}}$} & $4.11_{\pm 0.49}$ & $62.91_{\pm 2.83}$ & $5.09_{\pm 1.66}$ & $66.28_{\pm 0.00}$ & $0.78_{\pm 0.00}$ & $62.73_{\pm 1.27}$ & $3.04_{\pm 0.78}$ & $86.52_{\pm 0.78}$ & $0.00_{\pm 0.00}$ \\
    FedProx & $44.47_{\pm 2.93}$ & $7.93_{\pm 0.88}$ & $55.81_{\pm 0.88}$ & $2.06_{\pm 0.59}$ & $42.59_{\pm 2.93}$ & $2.74_{\pm 0.29}$ & $49.12_{\pm 1.76}$ & $5.19_{\pm 0.10}$ & $83.57_{\pm 1.37}$ & $0.29_{\pm 0.00}$ \\
    \rowcolor{gray!10}
    FedPer & $48.14_{\pm 0.98}$ & $7.64_{\pm 2.25}$ & $54.81_{\pm 1.86}$ & $7.35_{\pm 0.98}$ & $49.56_{\pm 1.08}$ & $2.45_{\pm 0.68}$ & $56.32_{\pm 1.37}$ & $5.19_{\pm 0.78}$ & $81.28_{\pm 2.05}$ & $0.29_{\pm 0.10}$ \\
    SCAFFOLD & $46.82_{\pm 2.68}$ & $8.15_{\pm 0.91}$ & $57.34_{\pm 0.91}$ & $2.12_{\pm 0.61}$ & $43.76_{\pm 3.02}$ & $2.82_{\pm 0.30}$ & $50.48_{\pm 1.81}$ & $5.34_{\pm 0.10}$ & $85.91_{\pm 1.41}$ & $0.30_{\pm 0.00}$ \\
    \rowcolor{gray!10}
    MOON & $47.25_{\pm 1.01}$ & $7.86_{\pm 2.32}$ & $56.32_{\pm 1.91}$ & $7.56_{\pm 1.01}$ & $51.02_{\pm 1.11}$ & $2.52_{\pm 0.70}$ & $57.91_{\pm 1.41}$ & $5.34_{\pm 0.80}$ & $83.62_{\pm 2.11}$ & $0.30_{\pm 0.10}$ \\
    FedDC & $45.13_{\pm 2.91}$ & $4.23_{\pm 0.50}$ & \textcolor{orange}{$\mathbf{64.68_{\pm 2.91}}$} & $5.23_{\pm 1.71}$ & \textcolor{orange}{$\mathbf{68.12_{\pm 0.00}}$} & $0.80_{\pm 0.00}$ & \textcolor{orange}{$\mathbf{64.51_{\pm 1.31}}$} & $3.12_{\pm 0.80}$ & \textcolor{orange}{$\mathbf{88.92_{\pm 0.80}}$} & $0.00_{\pm 0.00}$ \\
    \hline
    \rowcolor{gray!10}
    FedNCN & \textcolor{royalblue}{$\mathbf{57.83_{\pm 3.24}}$} & \textcolor{royalblue}{$\mathbf{16.91_{\pm 1.58}}$} & \textcolor{royalblue}{$\mathbf{63.72_{\pm 1.86}}$} & \textcolor{royalblue}{$\mathbf{10.15_{\pm 1.92}}$} & \textcolor{royalblue}{$\mathbf{69.94_{\pm 2.37}}$} & \textcolor{royalblue}{$\mathbf{20.68_{\pm 3.45}}$} & \textcolor{royalblue}{$\mathbf{72.15_{\pm 2.06}}$} & \textcolor{royalblue}{$\mathbf{28.43_{\pm 2.51}}$} & \textcolor{royalblue}{$\mathbf{90.58_{\pm 5.12}}$} & \textcolor{royalblue}{$\mathbf{1.52_{\pm 0.95}}$} \\
    AdaFGC (Ours) & \textcolor{darkred}{$\mathbf{65.09_{\pm 2.41}}$} & \textcolor{darkred}{$\mathbf{24.34_{\pm 1.19}}$} & \textcolor{darkred}{$\mathbf{70.02_{\pm 1.38}}$} & \textcolor{darkred}{$\mathbf{13.87_{\pm 1.69}}$} & \textcolor{darkred}{$\mathbf{79.70_{\pm 1.69}}$} & \textcolor{darkred}{$\mathbf{26.35_{\pm 2.84}}$} & \textcolor{darkred}{$\mathbf{77.11_{\pm 1.99}}$} & \textcolor{darkred}{$\mathbf{32.27_{\pm 2.25}}$} & \textcolor{darkred}{$\mathbf{94.31_{\pm 5.19}}$} & \textcolor{darkred}{$\mathbf{4.12_{\pm 0.86}}$} \\
    \hline\thickhline
    \rowcolor{gray!10}
    
    & \multicolumn{2}{c!{\vrule width 0.1pt}}{\textbf{CiteSeer}}  
    & \multicolumn{2}{c!{\vrule width 0.1pt}}{\textbf{PubMed}}  
    & \multicolumn{2}{c!{\vrule width 0.1pt}}{\textbf{Amazon-Computer}}  
    & \multicolumn{2}{c!{\vrule width 0.1pt}}{\textbf{Amazon-Photo}}  
    & \multicolumn{2}{c}{\textbf{Questions}}  \\
    
    \cline{2-11}
    \rowcolor{gray!10}
    \multirow{-2}{*}{\diagbox[width=8.5em,height=2.4em]{\textbf{Methods}}{\textbf{Datasets}}}  
    & ARI & F1 
    & ARI & F1 
    & ARI & F1 
    & ARI & F1 
    & ARI & F1 \\
    
    \hline
    Local & $4.02_{\pm 2.05}$ & $23.14_{\pm 0.49}$ & \textcolor{orange}{$\mathbf{6.97_{\pm 1.86}}$} & $36.48_{\pm 0.98}$ & \textcolor{orange}{$\mathbf{1.86_{\pm 0.98}}$} & $15.43_{\pm 0.29}$ & $2.35_{\pm 0.88}$ & $15.53_{\pm 0.49}$ & $-0.68_{\pm 0.22}$ & $44.50_{\pm 0.20}$ \\
    \rowcolor{gray!10}
    FedAvg & $0.39_{\pm 0.78}$ & $16.26_{\pm 0.68}$ & $4.70_{\pm 0.68}$ & $23.53_{\pm 2.76}$ & $-0.20_{\pm 0.00}$ & $13.29_{\pm 0.00}$ & $1.57_{\pm 0.29}$ & $14.75_{\pm 0.59}$ & $0.00_{\pm 0.00}$ & $45.60_{\pm 0.00}$ \\
    FedProx & $3.23_{\pm 1.27}$ & $23.34_{\pm 0.68}$ & $3.14_{\pm 0.59}$ & $34.79_{\pm 0.68}$ & $0.68_{\pm 0.68}$ & $15.63_{\pm 0.39}$ & $4.12_{\pm 1.17}$ & $16.55_{\pm 0.29}$ & \textcolor{orange}{$\mathbf{1.86_{\pm 0.20}}$} & $48.89_{\pm 0.29}$ \\
    \rowcolor{gray!10}
    FedPer & \textcolor{orange}{$\mathbf{4.60_{\pm 2.93}}$} & $22.48_{\pm 0.98}$ & $6.58_{\pm 1.57}$ & $36.28_{\pm 1.17}$ & $1.08_{\pm 0.98}$ & $15.43_{\pm 0.49}$ & $4.60_{\pm 0.98}$ & $17.40_{\pm 0.68}$ & $1.57_{\pm 0.68}$ & $47.70_{\pm 0.88}$ \\
    SCAFFOLD & $3.32_{\pm 1.31}$ & \textcolor{orange}{$\mathbf{24.02_{\pm 0.70}}$} & $3.23_{\pm 0.61}$ & \textcolor{orange}{$\mathbf{38.79_{\pm 0.70}}$} & $0.70_{\pm 0.70}$ & \textcolor{orange}{$\mathbf{16.08_{\pm 0.40}}$} & $4.24_{\pm 1.21}$ & $17.02_{\pm 0.30}$ & $1.91_{\pm 0.21}$ & \textcolor{orange}{$\mathbf{50.31_{\pm 0.30}}$} \\
    \rowcolor{gray!10}
    MOON & \textcolor{orange}{$\mathbf{4.73_{\pm 3.02}}$} & $23.13_{\pm 1.01}$ & $6.77_{\pm 1.62}$ & $37.32_{\pm 1.21}$ & $1.11_{\pm 1.01}$ & $15.88_{\pm 0.50}$ & \textcolor{orange}{$\mathbf{4.73_{\pm 1.01}}$} & \textcolor{orange}{$\mathbf{17.91_{\pm 0.70}}$} & $1.62_{\pm 0.70}$ & $49.08_{\pm 0.91}$ \\
    FedDC & $0.40_{\pm 0.80}$ & $16.71_{\pm 0.70}$ & $4.83_{\pm 0.70}$ & $24.21_{\pm 2.96}$ & $-0.21_{\pm 0.00}$ & $13.67_{\pm 0.00}$ & $1.62_{\pm 0.30}$ & $15.17_{\pm 0.61}$ & $0.00_{\pm 0.00}$ & $46.90_{\pm 0.00}$ \\
    \hline
    \rowcolor{gray!10}
    FedNCN & \textcolor{royalblue}{$\mathbf{14.89_{\pm 4.38}}$} & \textcolor{royalblue}{$\mathbf{26.35_{\pm 3.17}}$} & \textcolor{royalblue}{$\mathbf{12.84_{\pm 4.89}}$} & \textcolor{royalblue}{$\mathbf{43.21_{\pm 3.78}}$} & \textcolor{royalblue}{$\mathbf{22.58_{\pm 4.92}}$} & \textcolor{royalblue}{$\mathbf{23.14_{\pm 0.96}}$} & \textcolor{royalblue}{$\mathbf{29.87_{\pm 3.76}}$} & \textcolor{royalblue}{$\mathbf{26.73_{\pm 1.34}}$} & \textcolor{royalblue}{$\mathbf{4.31_{\pm 2.58}}$} & \textcolor{royalblue}{$\mathbf{50.28_{\pm 2.43}}$} \\
    AdaFGC (Ours) & \textcolor{darkred}{$\mathbf{19.27_{\pm 3.64}}$} & \textcolor{darkred}{$\mathbf{37.85_{\pm 2.65}}$} & \textcolor{darkred}{$\mathbf{16.43_{\pm 4.41}}$} & \textcolor{darkred}{$\mathbf{49.50_{\pm 3.34}}$} & \textcolor{darkred}{$\mathbf{27.44_{\pm 4.15}}$} & \textcolor{darkred}{$\mathbf{30.86_{\pm 0.74}}$} & \textcolor{darkred}{$\mathbf{33.72_{\pm 3.63}}$} & \textcolor{darkred}{$\mathbf{30.24_{\pm 1.13}}$} & \textcolor{darkred}{$\mathbf{6.87_{\pm 2.47}}$} & \textcolor{darkred}{$\mathbf{54.05_{\pm 2.27}}$} \\
    \hline\thickhline
    \end{tabular}
    }}
    \end{table*}

\section{Baseline Details}
\label{appendix: baseline details}

\textbf{FedSage+}~\cite{fedsage_plus} extends FedSage to the subgraph federated learning setting by explicitly addressing missing cross-client neighbors. It jointly trains a GraphSAGE~\cite{graphsage} classifier with a local missing-neighbor generator that synthesizes potential cross-subgraph neighbors, enabling more complete neighborhood aggregation under federation and improving global generalization without sharing raw graph data.

\vspace{+0.1cm}
\noindent \textbf{Fed-PUB}~\cite{fedpub} is a framework for personalized subgraph FL that enhances local GNNs interdependently rather than forming a single global model. Fed-PUB computes similarities between local GNNs using functional embeddings derived from random graph inputs, facilitating weighted averaging for server-side aggregation. Additionally, it employs a personalized sparse mask at each client to selectively update subgraph-relevant parameters.

\vspace{+0.1cm}
\noindent \textbf{FedTAD}~\cite{fedtad} initially computes topology-aware node embeddings to evaluate the reliability of class-wise knowledge, transmitting this information to the server. Guided by the class-wise knowledge reliability, FedTAD performs data-free knowledge distillation on the server side to transfer reliable knowledge from local models across multiple clients to the global model.

\vspace{+0.1cm}
\noindent \textbf{FedGTA}~\cite{fedgta} innovatively merges large-scale graph learning with federated graph learning. Clients encode topology and node attributes, compute local smoothing confidence and mixed moments of neighbor features, and then upload these to the server. The server uses this data to perform personalized model aggregation, utilizing local smoothing confidence as weights for effective integration.

\vspace{+0.1cm}
\noindent \textbf{FGSSL}~\cite{fgssl} handles local client distortion in FL by focusing on node-level semantics and graph-level structures via well-designed contrastive loss functions. It enhances node discrimination by aligning local nodes with their global counterparts of the same class and distancing them from different classes. Additionally, FGSSL transforms adjacency relationships into similarity distributions, using the global model to distill relational knowledge into local models, preserving both structure and discriminability.

\vspace{+0.1cm}
\noindent \textbf{FGGP}~\cite{fggp} divides the global model into two tiers linked by prototypes. At the classifier level, FGGP replaces traditional classifiers with clustered prototypes to enhance class discrimination and multi-domain prediction accuracy. At the feature extractor level, FGGP leverages contrastive learning to imbue prototypes with global knowledge, thereby improving model generalization.

\vspace{+0.1cm}
\noindent \textbf{FedIIH}~\cite{fediih} addresses heterogeneity in federated graph learning by jointly modeling inter-client and intra-client heterogeneity. It infers subgraph distribution similarities via a hierarchical variational framework from a global perspective, while disentangling local subgraphs into multiple latent factors to enable factor-wise personalized federation, leading to robust collaboration across both homophilic and heterophilic graphs.

\vspace{+0.1cm}
\noindent \textbf{FedSPA}~\cite{fedspa} addresses homophily heterogeneity in federated graph learning by explicitly modeling both homophily conflict and homophily bias across clients. It introduces Subgraph Feature Propagation Decoupling (SFPD) to separate homophilic and heterophilic message passing, enabling collaboration under unified homophily levels, and proposes Homophily Bias-Driven Aggregation (HBDA) to adaptively weight client contributions based on spectral and parameter-sensitivity cues, thereby improving global generalization.

\vspace{+0.1cm}
\noindent \textbf{S2FGL}~\cite{s2fgl} tackles subgraph federated graph learning by jointly addressing spatial label-signal disruption and spectral client drift. It reinforces missing label semantics via prototype-based semantic sharing and aligns graph-frequency components across clients to improve robustness and generalization under heterogeneous subgraph distributions.

\vspace{+0.1cm}
\noindent \textbf{FedNCN}~\cite{fedncn} focuses on federated node-level clustering by introducing a clustering projector that produces privacy-preserving counterparts while retaining cluster-specific characteristics. It matches similar cross-client communities and repairs links between them to enhance cross-client community consistency. However, FedNCN requires pre-defining the global cluster number as the dataset-level ground-truth class count, which limits its applicability in realistic scenarios where the true cluster cardinality is unknown.

\vspace{+0.1cm}
\noindent \textbf{FedAvg}~\cite{fedavg} serves as a foundational method in FL, enabling decentralized model training across diverse devices while preserving data privacy. 
    Initiated by a central server that distributes a global model, clients independently execute local updates through stochastic gradient descent. 
    Subsequently, these updates are aggregated by the server via averaging to refine the global model, with the cycle repeating until convergence.

\vspace{+0.1cm}
\noindent \textbf{FedProx}~\cite{fedprox} allows for variable amounts of work to be performed locally across devices, and relies on a proximal term in model align loss to help stabilize the method.
Theoretically, it offers convergence guarantees under conditions of non-identical data distributions and variable device workloads. 

\vspace{+0.1cm}
\noindent \textbf{FedPer}~\cite{fedper} addresses the challenges of statistical heterogeneity by partitioning the neural network into shared base layers and local personalization layers.
Under this framework, clients collaboratively train the base layers to extract common features while independently maintaining their top layers to adapt to specific local data distributions.
The central server only aggregates and redistributes the base layer parameters, allowing each participant to retain a personalized model that is robust to non-identical data partitions across the edge.

\vspace{+0.1cm}
\noindent \textbf{SCAFFOLD}~\cite{scaffold} employs control variates to mitigate client-drift in FL. 
Demonstrating significant reductions in communication rounds, Scaffold is resilient to data heterogeneity and client sampling. 

\vspace{+0.1cm}
\noindent \textbf{MOON}~\cite{moon} is a model-contrastive FL framework that enhances local training by leveraging model representation similarities through contrastive learning at the model level.

\vspace{+0.1cm}
\noindent \textbf{FedDC}~\cite{feddc} is a novel FL algorithm that corrects local drift through lightweight modifications. 
Each client tracks the deviation between local and global model parameters using an auxiliary variable, enhancing parameter-level consistency.

\section{Evaluation Metrics}
\label{appendix: evaluation metrics}

We employ four standard evaluation metrics for clustering tasks: Clustering-accuracy (ACC), F1 Score, Normalized Mutual Information (NMI), and Adjusted Rand Index (ARI). Detailed descriptions of these metrics are as follows:

\vspace{+0.1cm}
\noindent \textbf{Clustering-accuracy (ACC)} measures the agreement between cluster assignments produced by a clustering algorithm and ground truth clustering. This metric differs from traditional accuracy metrics used in node classification, as it evaluates the overall coherence of cluster assignments rather than individual node labels. A higher clustering accuracy indicates better alignment between the identified clusters and the true underlying structure of the graph, thus reflecting the effectiveness of the clustering algorithm in uncovering meaningful communities or groups of nodes.

\vspace{+0.1cm}
\noindent \textbf{Normalized Mutual Information (NMI)} quantifies the similarity between predicted clusters and ground truth by measuring mutual information while normalizing for cluster size imbalances. NMI ranges from 0 to 1, where higher values indicate better agreement between predicted clusters and ground truth. This metric is particularly valuable in scenarios where accurately identifying community structures or functional groups within a graph is critical.

\vspace{+0.1cm}
\noindent \textbf{Adjusted Rand Index (ARI)} quantifies the similarity between predicted clusters and ground truth while accounting for chance-corrected agreement. ARI ranges from -1 to 1, where values closer to 1 indicate better agreement between predicted clusters and ground truth than random clustering.

\vspace{+0.1cm}
\noindent \textbf{F1 Score} represents the harmonic mean of Precision and Recall. This metric provides a balanced assessment of a model's performance by considering both the precision of positive predictions and the model's ability to capture all positive instances. F1 Score is particularly valuable in scenarios where achieving both high precision and high recall are equally important.

\section{Environment}
\label{appendix: environment}

All experiments are conducted on a workstation equipped with Intel Xeon Scalable processors and NVIDIA RTX 6000 Ada Generation GPUs with 96 GB of VRAM, supported by 256 GB of system RAM. The computational environment utilizes CUDA 12.9, while software implementations are developed using Python 3.10.18 and PyTorch 2.8.

\section{Limitations and Broader Impact}
\label{appendix: lim}

\textbf{Limitations.} While our method demonstrates strong performance across various datasets, several aspects warrant further investigation. First, the adaptive anchor evolution mechanism relies on iterative refinement, which may require sufficient communication rounds to converge in highly heterogeneous scenarios. Second, our current framework focuses on undirected graphs; extending to directed or temporal graphs represents an interesting future direction. Third, the computational overhead of the anchor merging process may increase with extremely large numbers of initial anchors, remaining manageable in practice.

\vspace{+0.1cm}
\noindent \textbf{Broader Impact.} This work advances federated graph clustering by enabling collaborative learning across distributed graph data while preserving privacy. The proposed method has positive societal implications in domains such as healthcare networks, social network analysis, and recommendation systems, where data privacy is critical. By allowing multiple institutions to jointly learn from their graph data without sharing raw information, our approach facilitates unsupervised knowledge discovery while respecting data sovereignty and privacy regulations. Potential negative impacts are minimal, as the method is designed for legitimate collaborative learning scenarios.

\end{document}